\documentclass{article}

\usepackage[preprint]{neurips_2026}

\usepackage[utf8]{inputenc} 
\usepackage[T1]{fontenc}    
\usepackage{hyperref}       
\usepackage{url}            
\usepackage{booktabs}       
\usepackage{graphicx}       
\usepackage{subcaption}     
\usepackage{amsfonts}       
\usepackage{amsmath}        
\usepackage{amssymb}        
\usepackage{amsthm}         
\usepackage{nicefrac}       
\usepackage{microtype}      
\usepackage{xcolor}         
\usepackage{bm}             
\usepackage{mathtools}      

\usepackage{soul}

\newtheorem{theorem}{Theorem}[section]
\newtheorem{lemma}[theorem]{Lemma}
\newtheorem{corollary}[theorem]{Corollary}

\theoremstyle{definition}
\newtheorem{definition}[theorem]{Definition}
\newtheorem{assumption}[theorem]{Assumption}
\theoremstyle{remark}
\newtheorem{remark}[theorem]{Remark}

\DeclareMathOperator*{\argmin}{arg\,min}

\DeclareMathOperator{\tr}{tr}
\DeclareMathOperator{\Var}{Var}

\newcommand{\Reals}{\mathbb{R}}
\newcommand{\Expect}{\mathbb{E}}

\newcommand{\Ind}{\mathbf{1}}
\newcommand{\cH}{\mathcal{H}}
\newcommand{\cD}{\mathcal{D}}
\newcommand{\cG}{\mathcal{G}}
\newcommand{\cL}{\mathcal{L}}
\newcommand{\cN}{\mathcal{N}}
\newcommand{\cP}{\mathcal{P}}
\newcommand{\cS}{\mathcal{S}}
\newcommand{\bx}{\bm{x}}
\newcommand{\bw}{\bm{w}}
\newcommand{\by}{\bm{y}}
\newcommand{\be}{\bm{e}}
\newcommand{\bE}{\bm{E}}
\newcommand{\bW}{\bm{W}}
\newcommand{\bell}{\bm{\ell}}
\newcommand{\bdelta}{\bm{\delta}}
\newcommand{\beps}{\bm{\epsilon}}
\newcommand{\bSigma}{\bm{\Sigma}}
\newcommand{\btheta}{\bm{\theta}}
\newcommand{\boldeta}{\bm{\eta}}
\newcommand{\norm}[1]{\left\|#1\right\|}

\newcommand{\inner}[2]{\langle #1,\, #2 \rangle}

\title{Decomposing MXFP4 quantization error for LLM reinforcement learning: reducible bias, recoverable deadzone, and an irreducible floor}

\author{%
  Shi-Liang (Bruce) Wu\thanks{Equal contribution.} \quad
  Xiao-Can (Bruce) Li\footnotemark[1]\hspace{0.5em}\thanks{Corresponding Author Email: \texttt{hsiaotsan.li@alumni.utoronto.ca}} \quad
  Zheng Shen \\
  Huawei Canada
}

\begin{document}

\maketitle

\begin{abstract}
MXFP4 arithmetic can dramatically accelerate reinforcement learning (RL) post-training of large language models (LLMs), yet the quantization error introduces severe accuracy degradation. Existing work treats the quantization error as a monolithic noise term, missing the distinct mechanisms upon interpreting how quantization error damages training. We prove an exact three-way  decomposition of quantization error and show how each component dominates a distinct RL training pathway. Our theoretical and empirical analysis decomposes the MXFP4 quantization error into three additive components: \emph{scale bias} from power-of-two rounding, \emph{deadzone truncation} from zeroing small values, and \emph{grid noise} from rounding to the nearest 4-bit grid. 
Each component dominates a distinct RL failure mode: scale bias accumulates multiplicatively through the backward pass, affecting gradient accuracy; deadzone truncation degrades rollout quality; and grid noise raises the policy's entropy. We combine corrections that are RL failure mode-targeted but not component-exclusive: Macro-block scaling to reduce scale bias, Outlier Fallback  recovers deadzone entries, but also partially reduces scale bias induced error, and Adaptive Quantization Noise (AQN) for controlling the policy entropy.
 On Qwen2.5-3B dense and Qwen3-30B-A3B-Base mixture-of-experts model, the targeted corrections match BF16 accuracy within $-0.7$\,pp and \emph{exceed} BF16 by $+1.0$\,pp respectively.
\end{abstract}

\section{Introduction}
\label{sec:intro}

RL post-training turns base LLMs into capable reasoners~\citep{shao2024deepseekmath,yu2025dapo}, but the compute cost is steep: rollout generation alone can dominate the training budget. The MXFP4 format~\citep{rouhani2023ocp} has 4-bit elements with a shared E8M0 scale per block of 32, offering up to $4\times$ throughput and $4\times$ memory reduction, with native support on NVIDIA Blackwell, AMD MI350, Huawei Ascend 950, and other accelerators. However, simply replacing BF16 with MXFP4 during RL post-training produces a large accuracy gap. To restore the accuracy, quantization-aware training \cite{liu2024llmqat} and post-training quantization \cite{lin2024awq, frantar2023gptq} techniques are employed to mitigate this degradation. These works treat quantization error as a single noise source.

Different from how previous works view the quantization error, our novel perspective  decomposes the MXFP4 quantization error into three additive error components induced by: 
(1) \textit{Scale bias}: the E8M0 scale overshoots the ideal scale (E8M$\infty$) by ${\sim}44\%$ on average, an bias that accumulates multiplicatively across layers in backward pass, affecting the gradient accuracy. This error can be reduced by improving scale precision.
(2) \textit{Deadzone truncation}: absolute values below $\frac{0.5}{2}\div 6 = 1/24$ of the block absolute maximum are zeroed, pruning 9\% of weights and degrading rollout quality. The weight values zeroed by the deadzone are recoverable by a known technique called Outlier Fallback (OF) \cite{zhang2025fallback}.
(3) \textit{Grid noise}: the coarse E2M1 grid adds zero-mean rounding noise that raises the policy's effective temperature, increasing exploration. This component is \emph{invariant to scale precision}, as it only depends on weights and the E2M1 grid, not the block scale. Therefore, once the precision format is given, this error is irreducible.



\textbf{Contributions.}
\begin{enumerate}
\item We provide a structural decomposition of MXFP4 quantization error into three additive components: scale bias (from power-of-two rounding), deadzone truncation (from zeroing small values), and grid noise (from rounding to the 4-bit grid). We demonstrate that the grid component is invariant to scale precision, establishing an irreducible error floor even with optimal scaling.
\item A \textit{dominance analysis} mapping each error component to a specific RL failure mode: scale bias causes inaccurate gradient, deadzone causes rollout degradation in forward pass, and grid noise increases the policy entropy.
\item Rather than treating quantization error as monolithic noise, we introduce specific corrections for each failure mode: Macro-block scaling to mitigate scale bias, Outlier Fallback to address deadzone truncation, and Adaptive Quantization Noise to control policy entropy.
\item We validate our targeted corrections on both dense (Qwen2.5-3B) and mixture-of-experts (Qwen3-30B-A3B-Base) models during RL post-training. On dense, MBS+AQN+OF recovers BF16 accuracy within 0.7 percentage points; on MoE, AQN+MBS+OF exceeds the BF16 baseline by 1.0 percentage point.
\end{enumerate}

\section{Related work}
\label{sec:related}

\textbf{Post-training quantization (PTQ).}
Early LLM quantization focused on integer formats: LLM.int8()~\citep{dettmers2022llmint8} and ZeroQuant~\citep{yao2022zeroquant} demonstrated INT8 inference with round-to-nearest over weight groups. GPTQ~\citep{frantar2023gptq} pushed to INT4 via second-order weight adjustments, and AWQ~\citep{lin2024awq} introduced activation-aware scaling. SqueezeLLM~\citep{kim2024squeezellm} and SpQR~\citep{dettmers2024spqr} handle outliers with mixed dense-sparse representations. These methods assume uniform integer grids and relatively large group sizes; MXFP4's non-uniform E2M1 grid and small block size (32) change the quantization dynamics fundamentally, as our decomposition reveals.

\textbf{Rotation-based and extreme compression.}
QuIP~\citep{chee2023quip}, QuIP\#~\citep{tseng2024quipsharp}, and QTIP~\citep{tseng2024qtip} use incoherence processing (random rotations, Hadamard transforms, trellis codes) to push compression to 2 bits with theoretical guarantees. SmoothQuant~\citep{xiao2023smoothquant} redistributes activation outliers to weights for INT8 quantization, while QuaRot~\citep{ashkboos2024quarot} and SpinQuant~\citep{liu2025spinquant} apply rotation matrices to eliminate outliers entirely. More recent rotation variants include grouped sequency-arranged rotations~\citep{choi2025rgsr}, pairwise rotations for reasoning-LLM inference~\citep{liang2025paroquant}, and MR-GPTQ~\citep{egiazarian2026mrgptq} which fuses micro-rotation into GPTQ-style PTQ specifically for MXFP4/NVFP4 with an accompanying CUTLASS-based kernel library (QuTLASS). These techniques operate \emph{upstream} of quantization by reshaping the weight/activation distribution. In the language of our decomposition (Section~\ref{sec:decomposition}), they primarily reduce grid and deadzone error by making the input more amenable to the quantizer. Our corrections instead operate \emph{during or after} quantization and are complementary.

\textbf{Quantization-aware training (QAT).}
LLM-QAT~\citep{liu2024llmqat} performs data-free QAT by using the pre-trained model to generate training data. QLoRA~\citep{dettmers2023qlora} combines NF4 weight quantization with LoRA adapters for memory-efficient fine-tuning. QeRL~\citep{huang2025qerl} applies LoRA-based QAT specifically to RL post-training, demonstrating that quantization noise can enhance exploration. Our work differs from QeRL in two ways: we use full-parameter updates (no LoRA), providing cleaner causal evidence, and we ground the exploration benefit in our error decomposition. Specifically, the grid noise component acts as a natural entropy regularizer.

\textbf{Low-precision FP4/FP8 training.}
A growing line of work targets training (not just inference) in 4--8 bit floating-point formats. On the FP8 side, \citet{micikevicius2022fp8} introduced the E4M3/E5M2 formats and~\citet{fishman2025fp8scaling,hernandezcano2025fp8gemm} demonstrated trillion-token-scale FP8 LLM pre-training, identifying activation-outlier instability as the central obstacle. On the FP4 side, beyond the original MXFP4 spec~\citep{rouhani2023ocp} and the MXFP4-RL pre-training of~\citet{tseng2025trainingllmsmxfp4}, recent work targets native FP4 training (Quartet~\citep{castro2025quartet}, FP4 All The Way~\citep{chmiel2025fp4allway}), NVFP4 with finer scale precision~\citep{nvidia2025nvfp4,panferov2026quartet2,cook2025fouroversix}, and MoE-specific FP4 training on Hopper~\citep{zhang2026practicalfp4moe}. Outlier-Safe Pre-Training~\citep{park2025osp} attacks the same problem from the pre-training side by preventing outliers from forming in the first place. \citet{hao2025lowprectraining} provide a recent survey covering both fixed-point and floating-point low-precision training. Our contribution is orthogonal to these training-time format choices: we target RL post-training of an existing BF16 checkpoint and explain how the static MXFP4 error structure maps onto specific RL failure modes.

\begin{table}[t]
  \caption{Empirical error decomposition across model scales. All three models show the same structure: scale and grid are anti-correlated ($\cos \approx -0.66$); deadzone is exactly orthogonal to both; the grid component is ${\sim}71\%$ of total MSE and invariant to model scale.}
  \label{tab:decomp_empirical}
  \centering
  \small
  \begin{tabular}{lcccccc}
    \toprule
    Model & $\frac{\norm{\be^{\mathrm{scale}}}^2}{\norm{\be}^2}$ & $\frac{\norm{\be^{\mathrm{DZ}}}^2}{\norm{\be}^2}$ & $\frac{\norm{\be^{\mathrm{grid}}}^2}{\norm{\be}^2}$ & $\frac{2\langle\be^{\mathrm{scale}},\be^{\mathrm{grid}}\rangle}{\norm{\be}^2}$ & $\cos$ & \# tensors \\
    \midrule
    Qwen2.5-3B  & 1.725 & 0.026 & 0.703 & $-1.454$ & $-0.657$ & 252 \\
    Qwen3-30B-A3B-Base & 1.726 & 0.022 & 0.712 & $-1.460$ & $-0.658$ & 18\,624 \\
    \bottomrule
  \end{tabular}
\end{table}
\begin{figure}[t]
  \centering
  \includegraphics[width=0.65\linewidth]{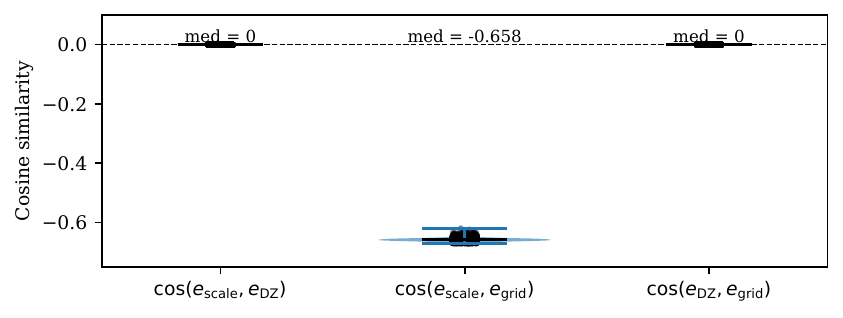}
  \caption{Pairwise error component cosine similarities across 18{,}624 weight tensors (Qwen3-30B-A3B-Base). DZ is exactly orthogonal to both scale and grid (point masses at 0); scale and grid are anti-correlated with $\cos \approx -0.66$ and minimal variance. This visualizes the structural relationships summarized in Table~\ref{tab:decomp_empirical}.}
  \label{fig:orthogonality}
\end{figure}
\begin{figure}[t]
  \centering
  \begin{subfigure}[t]{0.42\linewidth}
    \centering
    \includegraphics[width=\linewidth]{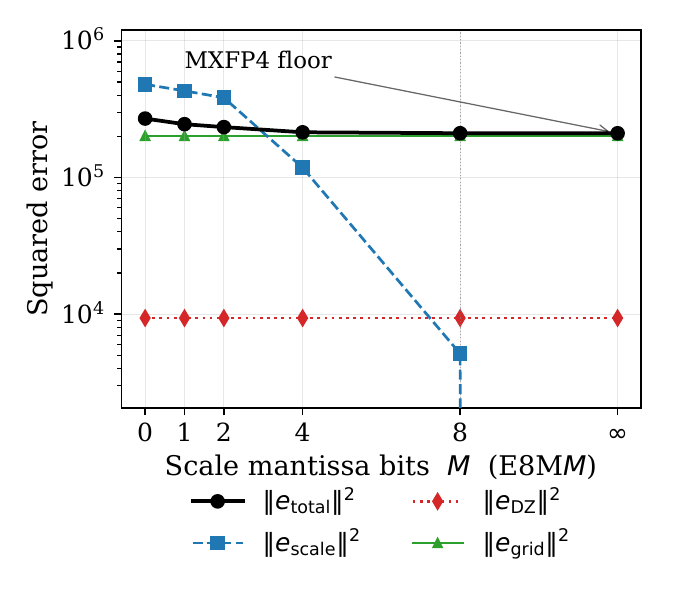}
    \caption{Scale mantissa precision sweep.}
    \label{fig:scale_sweep}
  \end{subfigure}
  \hfill
  \begin{subfigure}[t]{0.56\linewidth}
    \centering
    \includegraphics[width=\linewidth]{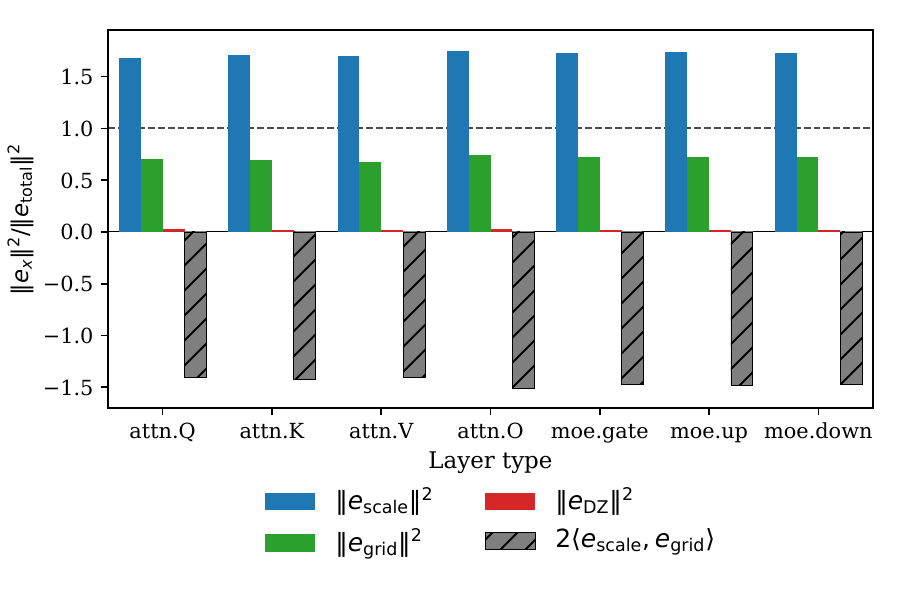}
    \caption{Per-layer-type error decomposition (Qwen3-30B-A3B-Base).}
    \label{fig:decomp}
  \end{subfigure}
  \caption{Error decomposition analysis. (a)~Improving scale precision drives total error to the irreducible floor $\norm{\be^{\mathrm{grid}}}^2 + \norm{\be^{\mathrm{DZ}}}^2$ without affecting grid error. (b)~All layer types share nearly same decomposition: $\norm{\be^{\mathrm{scale}}}^2 \approx 1.72\norm{\be}^2$, $\norm{\be^{\mathrm{grid}}}^2 \approx 0.71\norm{\be}^2$, $\norm{\be^{\mathrm{DZ}}}^2 \approx 0.02\norm{\be}^2$.}
  \label{fig:error_analysis}
\end{figure}

\section{MXFP4 quantization error decomposition}
\label{sec:decomposition}

\subsection{Preliminaries and decomposition}
\label{sec:prelim}

The MXFP4 format (OCP MX standard) quantizes a tensor $\bx\in\Reals^n$ block-wise with block size $B{=}32$. Each block $\bx_b = (x_{b,1},\ldots,x_{b,B})\in\Reals^B$ is quantized in two steps: \textbf{(i)} a shared E8M0 scale $s_b = 2^{\lceil\log_2(\max_i|x_{b,i}|/q_{\max})\rceil}$ and \textbf{(ii)} per-element rounding to the E2M1 grid $\cG = \{0, \pm0.5, \pm1, \pm1.5, \pm2, \pm3, \pm4, \pm6\}$ ($q_{\max}{=}6$), giving a \textit{dequantized} value $Q(x_{b,i}) = s_b \cdot \argmin_{g\in\cG} |x_{b,i}/s_b - g|$.

We introduce the \emph{ideal-scale quantizer} $Q^*$ that uses the unquantized scale $s_b^* = \max_i|x_{b,i}|/q_{\max}$ instead of the E8M0 $s_b$. This separates scale quantization from element-level quantization. The \emph{deadzone} $\cD_b$ is the set of elements satisfying $|x_{b,i}| < m_b/24$ (where $m_b = \max_j|x_{b,j}|$), which are rounded to zero; empirically 9\% of values per block.

\begin{definition}[Error components]
\label{def:error_components}
The total quantization error $e_{b,i} = Q(x_{b,i}) - x_{b,i}$ decomposes as
\begin{equation}
\label{eq:decomp}
  e_{b,i}
  = \underbrace{Q(x_{b,i}) - Q^*(x_{b,i})}_{e_{b,i}^{\mathrm{scale}}}
  + \underbrace{[Q^*(x_{b,i}) - x_{b,i}]\, \Ind_{\cD_b}(i)}_{e_{b,i}^{\mathrm{DZ}}}
  + \underbrace{[Q^*(x_{b,i}) - x_{b,i}]\, \Ind_{\cD_b^c}(i)}_{e_{b,i}^{\mathrm{grid}}}\,,
\end{equation}
where $\Ind_{\cD_b}(i) = \Ind[|x_{b,i}/s_b^*| < q_{\min}/2]$ is the deadzone indicator ($q_{\min}{=}0.5$). Since $\Ind_{\cD_b} + \Ind_{\cD_b^c} = 1$, the decomposition is exact: $\be = \be^{\mathrm{scale}} + \be^{\mathrm{DZ}} + \be^{\mathrm{grid}}$.
\end{definition}

\subsection{Interactions among three error components}
\label{sec:identity}

\begin{lemma}[Exact orthogonality: DZ $\perp$ Scale and DZ $\perp$ Grid]
\label{lem:exact}
$\inner{\be^{\mathrm{DZ}}}{\be^{\mathrm{scale}}} = \inner{\be^{\mathrm{DZ}}}{\be^{\mathrm{grid}}} = 0$ (pointwise, no assumptions needed).
\end{lemma}
\begin{proof}
$\be^{\mathrm{DZ}}$ and $\be^{\mathrm{grid}}$ have disjoint support by construction ($\Ind_{\cD_b}\cdot\Ind_{\cD_b^c} = 0$). For $\be^{\mathrm{DZ}} \perp \be^{\mathrm{scale}}$: on deadzone elements, $|x_i/s_b^*| < q_{\min}/2$, so $Q^*(x_i) = 0$. Since $s_b \geq s_b^*$ (ceiling rounding), $|x_i/s_b| \leq |x_i/s_b^*| < q_{\min}/2$, so $Q(x_i) = 0$ as well. Therefore $e_i^{\mathrm{scale}} = Q(x_i) - Q^*(x_i) = 0$ on every deadzone element.
\end{proof}

Expanding $\norm{\be}^2 = \norm{\be^{\mathrm{scale}} + \be^{\mathrm{DZ}} + \be^{\mathrm{grid}}}^2$ and applying Lemma~\ref{lem:exact} to eliminate two of three cross terms yields the MSE identity:
\begin{equation}
\label{eq:identity}
  \boxed{\norm{\be}^2
  = \norm{\be^{\mathrm{scale}}}^2
  + \norm{\be^{\mathrm{DZ}}}^2
  + \norm{\be^{\mathrm{grid}}}^2
  + 2\inner{\be^{\mathrm{scale}}}{\be^{\mathrm{grid}}}}\,,
\end{equation}
with exactly one surviving cross term.

\begin{remark}[Anti-correlation between scale and grid]
\label{rem:anticorr}
We observe $\cos(\be^{\mathrm{scale}}, \be^{\mathrm{grid}}) \approx -0.66$ consistently across 18{,}876 weight tensors in two model scales (Table~\ref{tab:decomp_empirical} and Figure \ref{fig:orthogonality}). The negative sign is expected: since $s_b \geq s_b^*$ (ceiling rounding), the quantized scale overshoots the ideal scale, meaning $Q$ uses a coarser grid than $Q^*$. Where $Q^*$ rounds a value upward (positive $e_i^{\mathrm{grid}}$), the coarser $Q$ grid tends to round to the same or a lower point (negative $e_i^{\mathrm{scale}}$), creating systematic sign opposition. The magnitude ${\approx}0.66$ is an empirical constant that we do not derive in closed form. Note that component norms sum to ${\approx}2.46\norm{\be}^2$, exceeding the total; the cross term $2\inner{\be^{\mathrm{scale}}}{\be^{\mathrm{grid}}} \approx -1.46\norm{\be}^2$ accounts for the difference. This has a practical consequence: MBS reduces both $\norm{\be^{\mathrm{scale}}}^2$ (172\% of total) \emph{and} the negative cross term, reducing total MSE to $\norm{\be^{\mathrm{grid}}}^2 + \norm{\be^{\mathrm{DZ}}}^2 \approx 0.73\norm{\be}^2$, a 27\% reduction.
\end{remark}

\subsection{Grid noise is invariant to scale precision}
\label{sec:invariance}

The grid error $e_i^{\mathrm{grid}}$ on non-deadzone elements depends only on the weight values $x_i$ and the ideal scale $s_b^*$, \emph{not on the actual E8M0 scale $s_b$}. Changing the scale precision (e.g., from E8M0 to E8M8 via MBS) modifies $\be^{\mathrm{scale}}$ but leaves $\be^{\mathrm{grid}}$ completely unchanged. As scale precision improves:
\begin{equation}
\label{eq:floor}
  \norm{\be^{\mathrm{scale}}} \to 0, \quad
  \inner{\be^{\mathrm{scale}}}{\be^{\mathrm{grid}}} \to 0, \quad
  \norm{\be^{\mathrm{total}}}^2 \to
  \underbrace{\norm{\be^{\mathrm{grid}}}^2 + \norm{\be^{\mathrm{DZ}}}^2}_{\text{irreducible floor}}\,.
\end{equation}
This floor is the structural limit of MXFP4: to go below it, the E2M1 grid itself must change. Figure~\ref{fig:scale_sweep} visualizes this convergence.

\section{Impact on RL training}
\label{sec:impact}

Each error component affects all three RL training pathways (policy gradient accuracy, rollout quality, and exploration--exploitation balance), but the mathematical structure of each creates a \textit{dominant} effect on one pathway. We analyze these below; the full $3 \times 3$ analysis is in Appendix~\ref{app:pathway_details}.

\begin{figure}[t]
  \centering
  \includegraphics[width=0.95\linewidth]{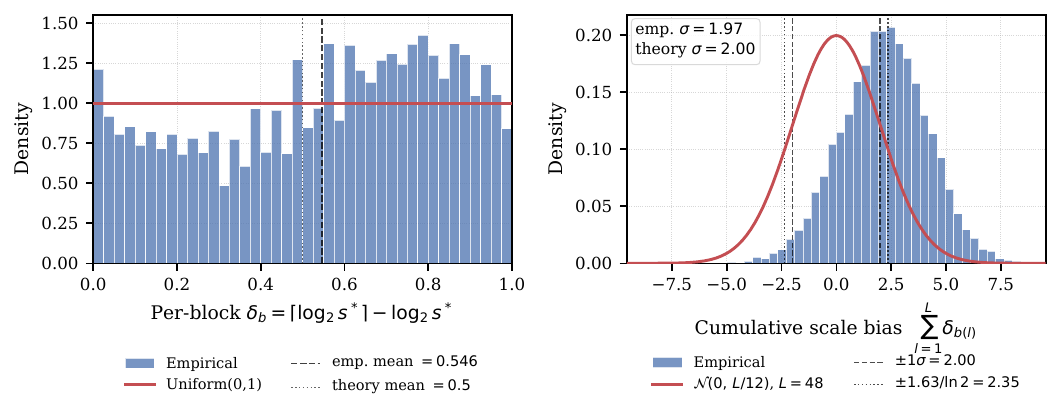}
  \caption{Scale bias from E8M0 scale rounding (Qwen3-30B-A3B-Base, $L{=}48$). (a)~Per-block rounding residual $\delta_b$ is approximately Uniform(0,1) with mean 0.546. (b)~Cumulative scale bias $\sum_l \delta_l$ across layers matches the theoretical standard deviation with empirical $\sigma = 1.97 \approx \sigma_{\mathrm{theory}} = 2.0$, with empirical mean near 2.0 due to imperfect Uniform distribution, validating the gradient inaccuracy in Section~\ref{sec:pathway_gradient}.}
  \label{fig:grad_fluct}
\end{figure}

\subsection{Scale bias: dominant effect on policy gradients}
\label{sec:pathway_gradient}

Scale bias is \emph{multiplicative}: each layer's output is scaled by $\gamma_b = s_b/s_b^*$. In the forward pass, LayerNorm resets activation magnitudes at every layer, preventing accumulation, so the per-layer forward-pass error is $O(\gamma_b {-} 1) \approx O(0.44)$, calculation shown in text after Eq. (\ref{eq:grad_log}) . In the backward pass, however, the straight-through estimator (STE) chain rule multiplies scale factors across all $L$ layers without any normalization. Let $\hat{\nabla}$ denote the gradient computed under quantization and $\nabla_{\mathrm{true}}$ the full-precision gradient. Their log magnitude ratio accumulates as:
\begin{equation}
\label{eq:grad_log}
  \log \frac{\norm{\hat{\nabla}}}{\norm{\nabla_{\mathrm{true}}}}
  \approx \sum_{l=1}^L \delta_{b(l)}\,,
  \quad \delta_b \sim \mathrm{Uniform}[0,\, 1)\,,
\end{equation}
where $\delta_b = \lceil\log_2 s_b^*\rceil - \log_2 s_b^*$ is the E8M0 ceiling-rounding error in log-space, approximately uniform by a Benford's-law argument (Appendix~\ref{app:proof_lemma2}, Step~1). Under ceiling rounding, $\gamma_b = 2^{\delta_b}$ with $\delta_b \sim \mathrm{Uniform} [0,1)$, giving $\mathrm{mean}(\gamma-1) = 1/\mathrm{ln}2 - 1 \approx 0.44$ and $\mathrm{RMSE}(\gamma-1) \approx 0.57$.
By the CLT, the centred sum $\sum_l (\delta_{b(l)} - \tfrac{1}{2}) \sim \cN(0, L/12)$. For Qwen2.5-3B with $L{=}36$ layers, $\mathrm{std} = \sqrt{3} \approx 1.73$, so the gradient magnitude ratio $\norm{\hat{\nabla}}/\norm{\nabla_{\mathrm{true}}}$ ranges from ${\sim}0.18\times$ to ${\sim}5.6\times$ within one standard deviation. This \textbf{exponential amplification in the backward pass}, absent in the forward pass, is what makes scale bias primarily a gradient problem. For the 48-layer MoE model (Qwen3-30B-A3B-Base), $\mathrm{std}_{\mathrm{theory}} = \sqrt{48/12} = 2.0$; we measure $\mathrm{std}_{\mathrm{emp}} = 1.97$ (Figure~\ref{fig:grad_fluct}). MBS corrects this by adding an 8-bit mantissa to the block scale, reducing $\Var(\gamma)$ by ${\sim}(256)^{-2}$.

\emph{Why not the other pathways?} Scale bias does perturb rollout logits and exploration at $O(\gamma{-}1) \approx 49\%$ per layer, but LayerNorm re-normalizes inputs between layers, preventing multiplicative compounding in the forward pass.

\subsection{Deadzone: dominant effect on rollout distribution}
\label{sec:pathway_reward}

Deadzone truncation is \emph{pruning}: it affects RL training by deterministically setting weights to zeros in the forward pass, changing the rollout distribution. On Qwen3-30B-A3B-Base model, $9.0\%$ of
 elements fall in this ideal deadzone and contribute $2.2\%$ of            
$\|e_{\rm total}\|^2$.

\emph{Why not gradients?} The STE backward pass \textbf{ignores the deadzone}: it passes gradients through as if the full-precision value were present. So deadzone damage is fully visible in the forward pass (degraded rollouts) but largely invisible in the backward pass (intact gradients). This asymmetry makes deadzone primarily a rollout quality problem.

\emph{Why not exploration?} Deadzone is \textbf{systematic and contractive} ($e_i^{\mathrm{DZ}} = -x_i$, always toward zero), not stochastic. It removes capacity rather than adding randomness, producing blander outputs (lower reward mean) rather than more diverse ones. In our dense experiments, OF recovers deadzone values (DZ rate: 9\% $\to$ 2\%), yielding $+$17.5\,pp.

\subsection{Grid noise: dominant effect on exploration}
\label{sec:pathway_exploration}

Grid noise is the only component that is \textbf{dense and approximately zero-mean}. A zero-mean perturbation to logits does not shift the mode of the output distribution; it \emph{widens} it, acting approximately as temperature scaling. Under additive noise $\boldeta \sim \cN(\bm{0}, \sigma_\eta^2\bm{I})$ on logits, the effective temperature is
\begin{equation}
\label{eq:eff_temp}
  T_{\mathrm{eff}} \approx \sqrt{1 + 2\sigma_\eta^2 / \Var(\Delta\ell)} > 1\,,
\end{equation}
where $\sigma_\eta^2 \propto \sum_l \sigma_{\delta,l}^2(\mathrm{grid})$ accumulates over layers. The noise level is \emph{constant} throughout training, providing no mechanism for annealing. Grid noise adds variance but no systematic bias to gradients (far milder than the $5.6\times$ scale-induced fluctuation), and its effect on reward quality averages out over full response sequences. Detailed explanation for Eq. (\ref{eq:eff_temp}) can be found in Appendix \ref{app:eff_temp}.

\section{Method}
\label{sec:method}

Guided by the decomposition in Section~\ref{sec:decomposition}, we adopt two error corrections and one training recipe. \textbf{MBS}~\citep{chhugani2026mxfp4} and \textbf{OF}~\citep{zhang2025fallback} are static error corrections: MBS reduces the scale component, OF recovers the deadzone component, together driving total error toward the irreducible grid floor. \textbf{AQN}~\citep{huang2025qerl} does not reduce any static error; it is an RL training recipe that helps the model learn effectively on this floor. We repurpose MBS and OF (originally designed for inference) for RL training with quantize-dequantize (QDQ) emulation, and adapt AQN with modifications for full-parameter training.

\textbf{QDQ emulation.}
To ensure compatibility, experiments are performed using simulated quantization/dequantization in PyTorch, the standard methodology for MXFP4 training research~\citep{tseng2025trainingllmsmxfp4,chen2025oscillation}, as native W4A4 tensor cores are not yet widely accessible. QDQ is numerically faithful: the quantization error is format-identical in both emulation and hardware-native execution (same E2M1 grid, same E8M0 block scale, FP32 accumulation). End-to-end latency on native accelerators is complementary future work that does not affect the accuracy results here.

\subsection{Macro Block Scaling (MBS): reducing scale bias toward the grid floor}
\label{sec:mbs}

Recall from Section~\ref{sec:prelim} that the E8M0 shared exponent $s_b = 2^e$ is a pure power of two, introducing a scale ratio $\gamma_b = s_b/s_b^*$ with ${\sim}54\%$ error. The root cause is that E8M0 has \emph{zero mantissa bits}: it can only represent scales at powers of two, wasting up to half the dynamic range.

MBS~\citep{chhugani2026mxfp4} compensates for this by computing an 8-bit mantissa correction at a coarser \emph{macro-block} granularity. Consider a macro-block of $B_M = 128$ elements, spanning four contiguous MXFP4 blocks of $B = 32$. The macro-block scale is factored as
\begin{equation}
\label{eq:mbs_factor}
  s^{\mathrm{MBS}} = 2^{e_M} \cdot (1 + m_{\mathrm{MBS}})\,,
\end{equation}
where $e_M$ is the macro-block exponent and $m_{\mathrm{MBS}} \in [0, 1)$ is an 8-bit mantissa (E0M8: 0 exponent bits, 8 mantissa bits, complementary to the E8M0 block scale; 256 levels). Since the per-block E8M0 scales already capture the exponent, MBS only needs to store the mantissa correction $m_{\mathrm{MBS}}$, a single byte per 128 elements, adding negligible memory overhead ($< 0.1$ bits/element).

\textbf{Prescale--quantize--postscale.}
In the QDQ emulation used during training, MBS wraps the standard MXFP4 quantization:
\begin{equation}
\label{eq:mbs_qdq}
  \hat{x}_i = \frac{1}{1 + m_{\mathrm{MBS}}} \cdot Q\!\bigl((1 + m_{\mathrm{MBS}}) \cdot x_i\bigr)\,,
\end{equation}
where $Q(\cdot)$ is the standard MXFP4 block quantizer (Definition~\ref{def:error_components}). The prescale $(1 + m_{\mathrm{MBS}})$ shifts the input distribution so that the power-of-two block scale $s_b$ more closely matches the ideal scale $s_b^*$; the postscale restores the original magnitude. In hardware GEMM, the prescale and postscale fuse into the tile-level epilogue at negligible cost~\citep{chhugani2026mxfp4}. The same adaptive-block-scaling idea has concurrently been proposed for NVFP4~\citep{cook2025fouroversix}, demonstrating that scale-precision correction is broadly useful across MX-family formats.

\textbf{Effect on error budget.}
With MBS, the effective scale ratio becomes $\gamma_b^{\mathrm{MBS}} = s_b / [s_b^* \cdot (1 + m_{\mathrm{MBS}})]$. The 8-bit mantissa reduces the scale quantization step from 1 (pure power-of-two) to $1/256$, shrinking $\Var(\gamma)$ by a factor of ${\sim}(256)^{-2} \approx 1.5 \times 10^{-5}$. By the grid-invariance property (Section~\ref{sec:invariance}), $\be^{\mathrm{grid}}$ is completely unaffected. As $\norm{\be^{\mathrm{scale}}} \to 0$, the cross term $\inner{\be^{\mathrm{scale}}}{\be^{\mathrm{grid}}} \to 0$ as well, and the total MSE converges to the irreducible floor $\norm{\be^{\mathrm{grid}}}^2 + \norm{\be^{\mathrm{DZ}}}^2$.

\subsection{Outlier Fallback (OF): recovering deadzone values}
\label{sec:of}

The deadzone truncates all values with $|x_i| < m_b/24$ to zero, destroying 9\% of values per block (Section~\ref{sec:pathway_reward}). Unlike scale bias, which can be corrected by adjusting the scale factor, deadzone is an \emph{information loss}: once a value is mapped to zero, no rescaling can recover it. OF addresses this through a two-pass residual quantization, adapting the block-level fallback technique of \citet{zhang2025fallback} from INT8 to MXFP4.

\textbf{Two-pass QDQ.}
Let $Q(\cdot)$ denote the standard MXFP4 block quantizer. OF replaces the single-pass $Q(\bx)$ with
\begin{equation}
\label{eq:of_result}
  \hat{\bx}_1 = Q(\bx),\quad
  \hat{\bx}_2 = Q(\bx - \hat{\bx}_1),\quad
  \hat{\bx}_{\mathrm{OF}} = \hat{\bx}_1 + \alpha\,\hat{\bx}_2,
\end{equation}
where $\alpha\in[0,1]$ blends the residual. Pass~1 quantizes outliers accurately (they set the block scale) but maps small values to the deadzone; Pass~2 has small dynamic range, so the previously dead values become representable. We use $\alpha=0.5$ throughout: $\hat{\bx}_2$ is itself lossy, and adding it at half strength outperforms $\alpha=1.0$ on GSM8K by ${\sim}1$\,pp and also minimizes rollout--training drift (U-shaped in $\alpha$ with minimum at $0.5$; Appendix~\ref{app:of_alpha}).

\subsection{Complementary RL training recipe: Adaptive Quantization Noise}
\label{sec:aqn}

After MBS and OF reduce the scale and deadzone components, the residual error is dominated by $\be^{\mathrm{grid}}$, the irreducible floor of the MXFP4 format. AQN~\citep{huang2025qerl} does not reduce this static error; it is a complementary RL training technique that helps the model learn effectively on the grid noise floor. The idea of \emph{controlled} parameter noise to aid exploration has classical roots in NoisyNet~\citep{fortunato2018noisynet} and gradient-noise injection~\citep{neelakantan2015adding}; AQN specialises this lineage to the structured grid-noise floor of MXFP4. Before each rollout, AQN perturbs weights with Gaussian noise $\beps \sim \cN(\bm{0}, \sigma^2\bm{I})$ whose magnitude decays exponentially from $\sigma_{\mathrm{start}}$ to $\sigma_{\mathrm{end}}$ over $K$ stages, preventing the premature entropy collapse caused by constant grid noise (Section~\ref{sec:pathway_exploration}). We use $\sigma_{\mathrm{start}}{=}1\%$, $\sigma_{\mathrm{end}}{=}0.1\%$, $K{=}10$, applying noise to both LayerNorm and linear layer weights.

\begin{figure}[t]
  \centering
  \includegraphics[width=0.9\linewidth]{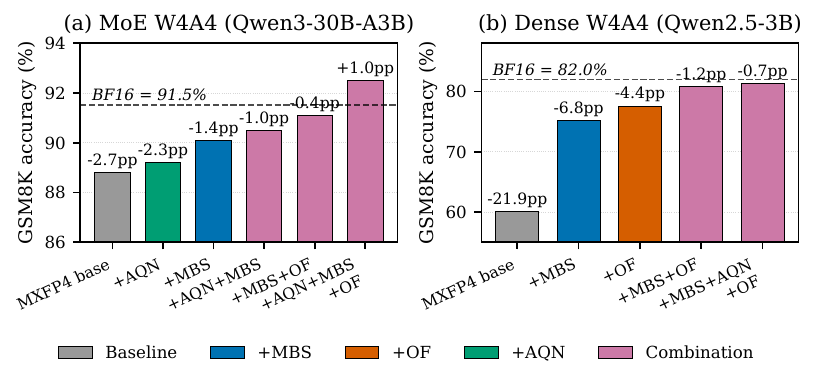}
  \caption{Ablation results on GSM8K. (a)~MoE: corrections stack incrementally; AQN+MBS+OF reaches $+$1.0\,pp (peak). (b)~Dense: OF is critical ($+$17.5\,pp alone); the full stack recovers to $-$0.7\,pp. Dashed line: BF16 baseline.}
  \label{fig:ablation}
\end{figure}

\section{Experiments}
\label{sec:experiments}

\subsection{Setup}
\label{sec:setup}

\textbf{Models and task.}
We evaluate on two architectures: Qwen2.5-3B (dense, 36 layers) and Qwen3-30B-A3B-Base (MoE, 48 layers, 3B active parameters). The task is GSM8K~\citep{cobbe2021gsm8k} math reasoning with verifiable rewards.

\textbf{Training.}
All runs use GRPO~\citep{shao2024deepseekmath} with batch size 64 and max response length 1024; the scaffolding builds on verl-VeRL alongside related RL-LLM frameworks~\citep{yu2025dapo,fu2025areal}. For MoE: Megatron parallelism, 2 samples per prompt. For dense: FSDP2, 4 samples per prompt. Full hyperparameter list in Appendix~\ref{app:hyperparams}.

\textbf{Quantization.}
MXFP4 W4A4 via QDQ emulation: both weights and activations quantized to E2M1 elements with E8M0 block scale (block size 32). Truncated Importance Sampling (TIS) \cite{yao2025tis} clips the importance ratio $\rho_i$ from above at $\rho_{\max}$ (i.e., $\min(\rho_i, \rho_{\max})$): we use TIS${}=2$ for MoE and TIS${}=5$ for dense. The larger threshold for dense is necessary because quantization-induced policy divergence concentrates importance ratios near the clipping boundary at TIS${}=2$, effectively zeroing out the learning signal (MBS and OF both cause training collapse at TIS${}=2$ on dense).

\begin{figure}[t]
  \centering
  \includegraphics[width=0.95\linewidth]{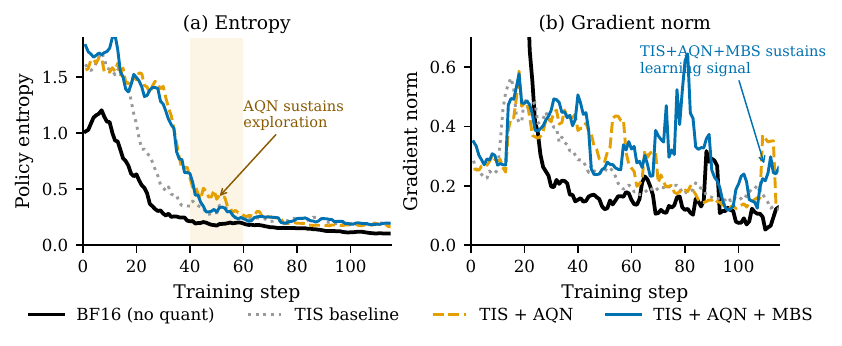}
  \caption{Training dynamics (MoE, GSM8K). (a)~AQN sustains policy entropy, preventing premature convergence. (b)~AQN+MBS maintains higher gradient norms, indicating continued learning. BF16 baseline shown for reference.}
  \label{fig:dynamics}
\end{figure}

\subsection{Main results}
\label{sec:main_results}

\begin{table}[t]
  \caption{MoE results (Qwen3-30B-A3B-Base, GSM8K, W4A4). BF16 baseline: 91.51\%.}
  \label{tab:moe}
  \centering
  \small
  \begin{tabular}{lccccc}
    \toprule
    Condition & AQN & MBS & OF & GSM8K (\%) & Gap \\
    \midrule
    MXFP4 baseline & & & & 88.8 & $-$2.7 \\
    +MBS & & \checkmark & & 90.1 & $-$1.4 \\
    +OF & & & \checkmark & 90.3 & $-$1.2 \\
    +AQN(1\%) & \checkmark & & & 89.2 & $-$2.3 \\
    +AQN+MBS & \checkmark & \checkmark & & 90.5 & $-$1.0 \\
    +MBS+OF & & \checkmark & \checkmark & 91.1 & $-$0.4 \\
    +AQN+MBS+OF & \checkmark & \checkmark & \checkmark & \textbf{92.49} & \textbf{$+$1.0} \\
    \bottomrule
  \end{tabular}
\end{table}

\begin{table}[!t]
  \caption{Dense results (Qwen2.5-3B, GSM8K, MXFP4 W4A4). BF16 baseline: 82.0\%.}
  \label{tab:dense}
  \centering
  \small
  \begin{tabular}{lccccc}
    \toprule
    Condition & MBS & AQN & OF & GSM8K (\%) & Gap \\
    \midrule
    MXFP4 baseline & & & & 60.1 & $-$21.9 \\
    +MBS & \checkmark & & & 75.2 & $-$6.8 \\
    +OF & & & \checkmark & 77.6 & $-$4.4 \\
    +MBS+OF & \checkmark & & \checkmark & 80.8 & $-$1.2 \\
    +MBS+AQN(1\%)+OF & \checkmark & \checkmark & \checkmark & \textbf{81.3} & $-$0.7 \\
    \bottomrule
  \end{tabular}
\end{table}

Tables~\ref{tab:moe}--\ref{tab:dense} present the main results. On the MoE model (BF16 baseline 91.51\%), AQN+MBS+OF ($\alpha\!=\!0.5$) reaches a peak of 92.49\%, exceeding BF16 by $+1.0$\,pp; MBS+OF (no AQN) closes the gap to $-0.4$\,pp (91.1\%). On dense, the na\"{\i}ve baseline suffers $-$21.9\,pp, but MBS+AQN+OF recovers to within 0.7\,pp of BF16 (81.3\% vs.\ 82.0\%). Figure~\ref{fig:ablation} visualizes these results.

 \textbf{Correction composition.}
  The three corrections compose near-additively on MoE: MBS alone
  $+$1.3\,pp, OF alone $+$1.5\,pp, AQN alone $+$0.4\,pp, versus
  $+$2.3\,pp combined. The slight subadditivity is consistent with
  partial overlap between MBS and OF on outlier-heavy blocks. On dense, the picture is strikingly
  different: OF alone ($+$17.5\,pp) dwarfs MBS ($+$15.1\,pp), and
  without OF, MBS recovers only 15\,pp of the 21.9\,pp gap.
  This asymmetry is predicted by the decomposition: dense models
  lack expert routing to compensate for deadzone pruning, making
  OF the critical correction.

\textbf{Dense vs.\ MoE dichotomy.}
OF provides $+$17.5\,pp on dense but only $+$1.5\,pp on MoE, because expert routing provides natural redundancy against deadzone pruning. Dense models have no such mechanism, making OF essential.

\textbf{AQN $\sigma$ sensitivity.}
On dense MXFP4 with MBS+OF, $\sigma_{\mathrm{start}} = 0.1\%$ is too weak (71.2\%), $2\%$ is slightly too strong (80.1\%), and $1\%$ is optimal (81.3\%). The optimal $\sigma_{\mathrm{start}}$ should be comparable to the per-element grid noise magnitude, which for MXFP4 is on the order of $1\%$ of the weight scale.

\textbf{Training dynamics.} \label{sec:dynamics}
Figure~\ref{fig:dynamics} shows that the MXFP4 baseline's entropy drops to 0.35 by step~50 (premature convergence), while AQN+MBS maintains 0.61. Gradient norms show the same pattern: baseline decays to 0.16 vs.\ 0.24 for AQN+MBS.

\section{Conclusion}
\label{sec:conclusion}
We presented a three-way decomposition of MXFP4 quantization error into scale bias, deadzone truncation, and grid noise, with distinct statistical signatures and distinct effects on RL training. Two formal results anchor it: (i)~exact orthogonality of the deadzone to both others (Lemma~\ref{lem:exact}), from the ceiling property $s_b \geq s_b^*$; and (ii)~invariance of grid noise to scale precision (Section~\ref{sec:invariance}), establishing an irreducible floor intrinsic to the E2M1 grid. The ratios $\|\mathbf{e}^{\mathrm{scale}}\|^2/\|\mathbf{e}\|^2 \approx 1.72$ and $\cos(\mathbf{e}^{\mathrm{scale}}, \mathbf{e}^{\mathrm{grid}}) \approx -0.66$ hold across 18,876 weight tensors in Qwen2.5-3B and Qwen3-30B-A3B-Base, suggesting they are properties of the format itself.

The corrections are mechanism-targeted but not component-exclusive: MBS reduces scale bias, OF recovers deadzone entries (and partially reduces scale-bias-induced error), and AQN controls residual grid noise. On Qwen2.5-3B (dense), MBS+AQN+OF recovers BF16 within 0.7\,pp (81.3\% vs.\ 82.0\%); on Qwen3-30B-A3B-Base (MoE), AQN+MBS+OF ($\alpha\!=\!0.5$) exceeds BF16 by $+1.0$\,pp (92.49\% vs.\ 91.51\%).

\textbf{Limitations.}\label{sec:limitations}
GSM8K uses short-to-medium responses ($1024$-token budget); long responses are harder to recover under the current recipe due to autoregressive trajectory divergence rather than averaging (Appendix~\ref{app:longcot}). Other reasoning, coding, and safety benchmarks remain to be validated. We use QDQ emulation, so Huawei Ascend 950 and other accelerators' throughput numbers are future work. The AQN $\sigma_{\mathrm{start}}$ requires per-model tuning and we do not yet provide an adaptive rule. Most configurations are single-seed.

\textbf{Future work.}
Combining MBS/OF with upstream techniques (stochastic rounding, Hadamard transforms~\citep{tseng2025trainingllmsmxfp4}) could further reduce the grid+deadzone floor. Extending the analysis to activation quantization, where deadzone rates are higher, would complete the W4A4 picture.


{
\small
\bibliographystyle{plainnat}
\bibliography{references}
}


\appendix

\section{Analysis of the cross term \texorpdfstring{$\langle \be^{\mathrm{scale}},\, \be^{\mathrm{grid}} \rangle$}{<e\^scale, e\^grid>}}
\label{app:proof_lemma2}

\textbf{Note on scope.} The main body establishes the MSE identity (Eq.~\ref{eq:identity}) and reports the empirical anti-correlation $\cos(\be^{\mathrm{scale}}, \be^{\mathrm{grid}}) \approx -0.66$ (Table~\ref{tab:decomp_empirical}). This appendix analyzes the cross term under the following idealized assumption. Under this assumption the cross term is $\mathcal{O}(B^{-1/2})$, but in practice, the ceiling-based scale selection ($\gamma_b \geq 1$ always) creates a structural anti-correlation absent from the idealized model. The analysis below characterizes the idealized setting for completeness; the empirical measurements in Table~\ref{tab:decomp_empirical} and Appendix~\ref{app:empirical_orth} characterize the real setting.

\begin{assumption}
\label{assum:A}
Within each block the elements $(x_{b,1}, \ldots, x_{b,B})$ are i.i.d.\ from a continuous, symmetric distribution $\cP$ with finite fourth moment. The quantity $\log_2 s_b^*$ is non-lattice.
\end{assumption}

Define the scale ratio $\gamma_b = s_b/s_b^* = 2^{\delta_b}$, normalized elements $u_{b,i} = q_{\max}\cdot x_{b,i}/m_b$, and let $r(\cdot)$ denote $\mathrm{round}_{\mathrm{E2M1}}$.

\textbf{Step 1: Characterize $\gamma_b$.}\;
Under Assumption~\ref{assum:A}, $\delta_b = \lceil\log_2 s_b^*\rceil - \log_2 s_b^*$ depends only on the fractional part of $\log_2 m_b$. For ceiling rounding, $\delta_b \in [0, 1)$. By the non-lattice condition, $\delta_b$ is approximately $\mathrm{Uniform}[0, 1)$, giving
\begin{equation}
\label{eq:app_gamma_dist}
  \gamma_b \sim 2^U,\; U \sim \mathrm{Uniform}[0,\, 1)\,,
  \quad
  \Expect[\gamma_b] = \frac{1}{\ln 2} \approx 1.44\,,
  \quad
  \mathrm{RMSE}(\gamma_b {-} 1) \approx 0.57\,.
\end{equation}
Empirically, we measure $\mathrm{RMSE}(\gamma_b {-} 1) = 0.57$ on Qwen3-30B-A3B-Base, consistent with the ceiling-rounding model.

\textbf{Step 2: Conditional independence structure.}\;
Write the cross term for block~$b$:
\begin{equation}
  \inner{\be_b^{\mathrm{scale}}}{\be_b^{\mathrm{elem}}}
  = \sum_{i=1}^B e_{b,i}^{\mathrm{scale}} \cdot e_{b,i}^{\mathrm{elem}}\,,
\end{equation}
where $e^{\mathrm{elem}} = e^{\mathrm{DZ}} + e^{\mathrm{grid}}$. Conditioning on $m_b$, the scale ratio $\gamma_b$ becomes deterministic:
\begin{align}
  e_{b,i}^{\mathrm{scale}} &= s_b \cdot r(x_{b,i}/s_b) - s_b^* \cdot r(x_{b,i}/s_b^*)
                           = h(\gamma_b,\, u_{b,i})\,, \\
  e_{b,i}^{\mathrm{elem}}  &= s_b^* [r(u_{b,i}) - u_{b,i}]
                           = s_b^*\, g(u_{b,i})\,,
\end{align}
with $h(\gamma_b, u) = 0$ when $\gamma_b = 1$.

\textbf{Step 3: Conditional distribution of non-max elements.}\;
Since the $B$ elements are i.i.d.\ from $\cP$, conditioning on the block maximum $m_b = \max_j |x_{b,j}|$ only tells us that each non-max element satisfies $|x_{b,i}| \leq m_b$. By Bayes' rule, the conditional distribution of each non-max element is simply $\cP$ truncated to $[-m_b, m_b]$:
\begin{equation}
  P(x_{b,i} \in A \mid m_b) = \frac{P(x_{b,i} \in A \cap [-m_b, m_b])}{F(m_b) - F(-m_b)}\,,
\end{equation}
and the $B{-}1$ non-max elements remain mutually independent (conditioning on the max restricts their range but does not introduce dependence). For symmetric $\cP$ with smooth density $f$, the conditional expectation of any function of the normalized element $u_{b,i} = q_{\max} \cdot x_{b,i}/m_b$ is:
\begin{equation}
  \Expect[g(u_{b,i}) \mid m_b]
  = \int_{-q_{\max}}^{q_{\max}} g(u)\,
    \frac{f(u \cdot m_b/q_{\max})}{F(m_b) - F(-m_b)}
    \cdot \frac{m_b}{q_{\max}}\, du\,.
\end{equation}

For $\gamma_b$ near~1, Taylor expansion gives $h(\gamma_b, u) \approx (\gamma_b - 1)\cdot\phi(u)$, so:
\begin{equation}
  \Expect[h(\gamma_b, u_{b,i})\cdot g(u_{b,i}) \mid m_b]
  \approx (\gamma_b - 1)\,\Expect[\phi(u_{b,i})\,g(u_{b,i}) \mid m_b]\,.
\end{equation}

\textbf{Step 4: CLT averaging.}\;
Taking expectation over $m_b$:
\begin{equation}
  \Expect[e_{b,i}^{\mathrm{scale}} \cdot e_{b,i}^{\mathrm{elem}}]
  = \Expect_{m_b}[(\gamma_b - 1)\,\psi(m_b)]\,,
\end{equation}
where $\psi(m_b) = \Expect[\phi(u)\,g(u) \mid m_b]$. The block-summed cross term has $B$ conditionally i.i.d.\ terms, so by the CLT:
\begin{equation}
  \frac{1}{B}\sum_{i=1}^B e_{b,i}^{\mathrm{scale}} \cdot e_{b,i}^{\mathrm{elem}}
  = (\gamma_b - 1)\,\psi(m_b) + \mathcal{O}_P(B^{-1/2})\,.
\end{equation}

The leading term $\Expect[(\gamma{-}1)\psi]$ is small: $\psi$ involves the product $\phi(u)\cdot g(u)$, which has alternating sign across E2M1 grid intervals. Under the smoothness of $f$, these cancel to $\mathcal{O}(\Var(\gamma)/\sqrt{B})$. Summing over $N$ blocks and normalizing by $\Expect[\norm{\be}^2] = \Theta(NB\cdot\Var(\gamma))$ yields the $\mathcal{O}(B^{-1/2})$ relative bound. $\blacksquare$

\section{Individual bounds for Scale $\perp$ DZ and Scale $\perp$ Grid}
\label{app:individual_bounds}

This appendix illustrates the orthogonality $\inner{\be^{\mathrm{scale}}}{\be^{\mathrm{DZ}}} = 0$ (Lemma~\ref{lem:exact}) and the anti-correlation $\inner{\be^{\mathrm{scale}}}{\be^{\mathrm{grid}}} < 0$ (Remark~\ref{rem:anticorr}) with a worked numerical example.

Because $\be^{\mathrm{DZ}}$ and $\be^{\mathrm{grid}}$ have disjoint support (Lemma~\ref{lem:exact}), the single surviving cross term partitions element-wise:
\begin{align}
  \inner{\be^{\mathrm{scale}}}{\be^{\mathrm{DZ}}}
  &= \sum_{b}\sum_{i \in \cD_b} e_{b,i}^{\mathrm{scale}} \cdot e_{b,i}^{\mathrm{DZ}}\,, \\
  \inner{\be^{\mathrm{scale}}}{\be^{\mathrm{grid}}}
  &= \sum_{b}\sum_{i \notin \cD_b} e_{b,i}^{\mathrm{scale}} \cdot e_{b,i}^{\mathrm{grid}}\,.
\end{align}
These are sums over \textbf{disjoint index sets} within each block. Two consequences follow.

\paragraph{Why $\inner{\be^{\mathrm{scale}}}{\be^{\mathrm{DZ}}} = 0$ exactly.}
For ceiling-based exponent selection ($s_b = 2^{\lceil\log_2 s_b^*\rceil}$, the standard convention to avoid overflow), $\gamma_b \geq 1$, so the deadzone under $s_b$ is \emph{wider} than under $s_b^*$. Every element in $\cD_b$ maps to zero under both scales, giving $e_i^{\mathrm{scale}} = Q(x_i) - Q^*(x_i) = 0$ for all $i \in \cD_b$. The only surviving cross term in the MSE identity (Eq.~\ref{eq:identity}) is therefore $\inner{\be^{\mathrm{scale}}}{\be^{\mathrm{grid}}}$, which acts exclusively on non-deadzone elements.

\paragraph{Numerical example.}
Consider a block $\bx = [0.03,\; 0.1,\; 0.3,\; 0.5,\; 0.9,\; 1.5,\; 2.0,\; 4.0]$ with $q_{\max}{=}6$. The ideal scale is $s^* = 4.0/6 = 0.667$; the E8M0 scale is $s = 2^0 = 1.0$; $\gamma = s/s^* = 1.5$. The deadzone threshold is $s^* \times 0.25 = 0.167$, placing $x_1{=}0.03$ and $x_2{=}0.1$ in $\cD_b$.

\begin{table}[h]
\centering
\small
\begin{tabular}{cccccccc}
\toprule
$i$ & $x_i$ & DZ? & $Q^*(x_i)$ & $Q(x_i)$ & $e_i^{\mathrm{scale}}$ & $e_i^{\mathrm{DZ}}$ & $e_i^{\mathrm{grid}}$ \\
\midrule
1 & 0.03 & \checkmark & 0     & 0    & \textbf{0}      & $-0.03$  & 0 \\
2 & 0.10 & \checkmark & 0     & 0    & \textbf{0}      & $-0.10$  & 0 \\
3 & 0.30 &            & 0.333 & 0.50 & $+0.167$ & 0        & $+0.033$ \\
4 & 0.50 &            & 0.667 & 0.50 & $-0.167$ & 0        & $+0.167$ \\
5 & 0.90 &            & 1.000 & 1.00 & $0$      & 0        & $+0.100$ \\
6 & 1.50 &            & 1.333 & 1.50 & $+0.167$ & 0        & $-0.167$ \\
7 & 2.00 &            & 2.000 & 2.00 & $0$      & 0        & $0$ \\
8 & 4.00 &            & 4.000 & 4.00 & $0$      & 0        & $0$ \\
\bottomrule
\end{tabular}
\end{table}

Since $\gamma \geq 1$, both deadzone elements have $e_i^{\mathrm{scale}} = 0$ (bold), giving $\inner{\be^{\mathrm{scale}}}{\be^{\mathrm{DZ}}} = 0$ exactly. The only nonzero cross term is $\inner{\be^{\mathrm{scale}}}{\be^{\mathrm{grid}}} = (+0.167)(+0.033) + (-0.167)(+0.167) + (+0.167)(-0.167) = -0.050$, which has mixed-sign terms that partially cancel within themselves. There is no opposing $\inner{\be^{\mathrm{scale}}}{\be^{\mathrm{DZ}}}$ to mask this value.

The numerical example also illustrates the anti-correlation: at $i = 3$, $e^{\mathrm{scale}} = +0.167$ and $e^{\mathrm{grid}} = +0.033$ (same sign, weak); at $i = 4$, $e^{\mathrm{scale}} = -0.167$ and $e^{\mathrm{grid}} = +0.167$ (opposite sign, strong); at $i = 6$, $e^{\mathrm{scale}} = +0.167$ and $e^{\mathrm{grid}} = -0.167$ (opposite sign). The net inner product is negative ($-0.050$), consistent with the $\cos \approx -0.66$ observed at scale across all weight tensors (Table~\ref{tab:decomp_empirical}).

\section{GEMM output error propagation}
\label{app:gemm}

For a linear layer $\by = \bW\bx$ with quantized weights $Q(\bW) = \bW + \bE$:
\begin{equation}
  \hat{\by} = Q(\bW)\bx = \by + \bE\bx\,,
  \quad \bE = \bE^{\mathrm{scale}} + \bE^{\mathrm{DZ}} + \bE^{\mathrm{grid}}\,.
\end{equation}

\begin{corollary}[GEMM-level decomposition]
\label{cor:gemm}
If $\bE$ and $\bx$ are independent, then $\Expect[\norm{\bE\bx}^2] = \tr(\Expect[\bE^\top\bE]\,\bSigma_x)$. By Eq.~\ref{eq:identity} and Lemma~\ref{lem:exact}, $\Expect[\bE^\top\bE]$ decomposes with one cross term (scale$\times$grid) and exact orthogonality of the DZ component:
\begin{equation}
  \Expect[\norm{\bdelta}^2]
  \approx
  \underbrace{\tr\!\bigl(\Expect[(\bE^{\mathrm{scale}})^\top\bE^{\mathrm{scale}}]\,\bSigma_x\bigr)}_{\sigma_\delta^2(\mathrm{scale})}
  + \underbrace{\tr\!\bigl(\Expect[(\bE^{\mathrm{DZ}})^\top\bE^{\mathrm{DZ}}]\,\bSigma_x\bigr)}_{\sigma_\delta^2(\mathrm{DZ})}
  + \underbrace{\tr\!\bigl(\Expect[(\bE^{\mathrm{grid}})^\top\bE^{\mathrm{grid}}]\,\bSigma_x\bigr)}_{\sigma_\delta^2(\mathrm{grid})}\,.
\end{equation}
\end{corollary}

The approximation drops the scale$\times$grid cross term $2\tr(\Expect[(\bE^{\mathrm{scale}})^\top\bE^{\mathrm{grid}}]\,\bSigma_x)$; under MBS ($\norm{\bE^{\mathrm{scale}}} \to 0$) this term vanishes and the equality becomes exact.

For W4A4 (both weights and activations quantized), the activation error $\be_A$ adds an independent term $\norm{\bW\be_A}^2$ (the cross term $\bE_W\be_A$ is second-order). The same three-way decomposition applies to $\be_A$.

\section{Detailed RL pathway analysis}
\label{app:pathway_details}

This appendix provides (i)~the full $3 \times 3$ analysis of how each error component interacts with each RL pathway, and (ii)~extended derivations for each dominant effect.

\subsection{Full $3 \times 3$ interaction analysis}
\label{app:full_matrix}

\paragraph{Scale bias on reward signal (secondary).}
Scale error perturbs each layer's output by a multiplicative factor $\gamma_b$, shifting rollout logits. However, LayerNorm resets activation magnitudes at each layer boundary, limiting the per-layer forward-pass distortion to $O(\gamma_b {-} 1) \approx O(0.5)$. This is significant but does not compound across layers the way it does in the backward pass.

\paragraph{Scale bias on exploration (secondary).}
The multiplicative logit shift also changes the effective entropy, but by the same $O(\gamma_b{-}1)$ per-layer amount. Since it has nonzero mean ($\Expect[\gamma{-}1] \approx 0.44$ under ceiling rounding), the effect is more bias-like than noise-like, contributing to mode shift rather than distribution widening.

\paragraph{Deadzone on gradients (secondary).}
The STE backward pass treats $Q$ as identity: gradients flow through deadzone elements as if they retained their original values. The deadzone therefore has minimal direct impact on gradient computation, though the degraded forwardpass indirectly changes the rollout distribution.

\paragraph{Deadzone on exploration (secondary).}
Deadzone is contractive ($e_i^{\mathrm{DZ}} = -x_i$, always toward zero), not stochastic. It reduces the model's representational capacity without adding randomness. The entropy may increase slightly (a less capable model produces more uniform predictions), but this is a side effect of capacity loss, not a controlled exploration mechanism.

\paragraph{Grid noise on reward signal (secondary).}
Grid noise adds zero-mean variance to each token's logit vector. Over a full response sequence of hundreds of tokens, these i.i.d.\ perturbations average out in their effect on the overall response quality. The reward is determined by the complete sequence, not individual tokens, so the impact on reward signal quality is attenuated by $O(1/\sqrt{T})$ where $T$ is the sequence length.

\paragraph{Grid noise on gradients (secondary).}
Zero-mean forward-pass noise produces zero-mean gradient noise. This adds variance to the gradient estimate but no systematic bias. In contrast, scale error introduces a multiplicative bias that grows exponentially with depth. The gradient variance from grid noise is comparable to the inherent stochasticity of mini-batch sampling and does not qualitatively change training dynamics.

\subsection{Dominant effect 1: deadzone and reward signal degradation}

\paragraph{Effective rank reduction.}
Define $r_{\mathrm{eff}}(\bW) = \norm{\bW}_*^2/\norm{\bW}_F^2$. The deadzone applies a soft-thresholding operator $\cS_t(\bW)$ that zeroes elements below threshold $t = m_b/24$. Since small singular values are disproportionately affected by element-wise thresholding (by the Eckart--Young theorem):
\begin{equation}
  r_{\mathrm{eff}}(\cS_t(\bW)) < r_{\mathrm{eff}}(\bW)\,.
\end{equation}

\paragraph{Advantage variance increase.}
Lower effective rank leads to rollout responses with less nuanced reasoning. The reward signal $r_i$ has lower mean and higher variance, making the group-normalized advantage $\hat{A}_i = (r_i - \bar{r})/\sigma_r$ noisier:
\begin{equation}
  \Var(\hat{A}_i^Q) > \Var(\hat{A}_i)\,.
\end{equation}

\paragraph{Why OF helps.}
OF's two-pass residual QDQ recovers deadzone values: Pass~1 quantizes normally (outliers are accurate, small values enter the deadzone); Pass~2 quantizes the residual (whose dynamic range is small enough for previously dead values to be represented). The DZ rate drops from 9\% to 2\%, directly restoring effective rank and rollout quality. In our dense experiments, OF alone provides $+$17.5\,pp recovery (60.1\% $\to$ 77.6\%).

\paragraph{OF residual-blend coefficient $\alpha$.}
\label{app:of_alpha}
The OF reconstruction $\hat{\bx}_{\mathrm{OF}} = \hat{\bx}_1 + \alpha\,\hat{\bx}_2$ (Eq.~\ref{eq:of_result}) includes a blend coefficient $\alpha$. The naive choice $\alpha=1$ adds the full residual back, but $\hat{\bx}_2$ is itself an MXFP4-quantized approximation of the residual and inherits its own grid + deadzone error. Adding it at half strength avoids over-correction and improves rollout quality on GSM8K (Table~\ref{tab:of_alpha}).
\begin{table}[h]
  \caption{OF residual-blend coefficient $\alpha$ sensitivity (Qwen3-30B-A3B-Base, GSM8K, W4A4 with MBS+AQN). $\alpha=0.5$ (used throughout the paper) outperforms $\alpha=1.0$ by ${\sim}1$\,pp.}
  \label{tab:of_alpha}
  \centering
  \small
  \begin{tabular}{lcc}
    \toprule
    $\alpha$ & GSM8K final (\%) & GSM8K peak (\%) \\
    \midrule
    $0.5$ (default)   & \textbf{91.58} & \textbf{92.49} \\
    $1.0$             & 90.45 & 90.45 \\
    \bottomrule
  \end{tabular}
\end{table}

\paragraph{Why $\alpha=0.5$: rollout--training drift minimization.}
The static argument above explains why $\alpha=1$ is suboptimal, but not why the optimum sits near one-half. We find a sharper, training-specific reason: $\alpha=0.5$ minimizes the rollout-vs-training numerical drift that RL is sensitive to. In a controlled single-step isolation run (Qwen3-30B-A3B, OF applied identically to weights and activations, MBS/AQN disabled, learning rate $0$, $32$ fixed DAPO-MATH prompts so $\alpha$ is the only varying quantity), the divergence between the vLLM rollout policy and the training-side recomputed policy is U-shaped in $\alpha$ with a clean minimum at $\alpha=0.5$ (Figure~\ref{fig:alpha_drift}); both extremes incur roughly $1.5$--$2\times$ the drift of $\alpha=0.5$. The interpretation is that a half-strength residual best aligns the rollout fake-quantization grid with the training grid: $\alpha<0.5$ under-recovers the deadzone, while $\alpha>0.5$ over-corrects per-channel and amplifies the mismatch between the rollout and training kernels. Since RL post-training is drift-sensitive, the $\alpha$ that minimizes this mismatch is also the $\alpha$ that maximizes downstream accuracy, consistent with the GSM8K gap in Table~\ref{tab:of_alpha}.
\begin{figure}[h]
  \centering
  \includegraphics[width=0.7\linewidth]{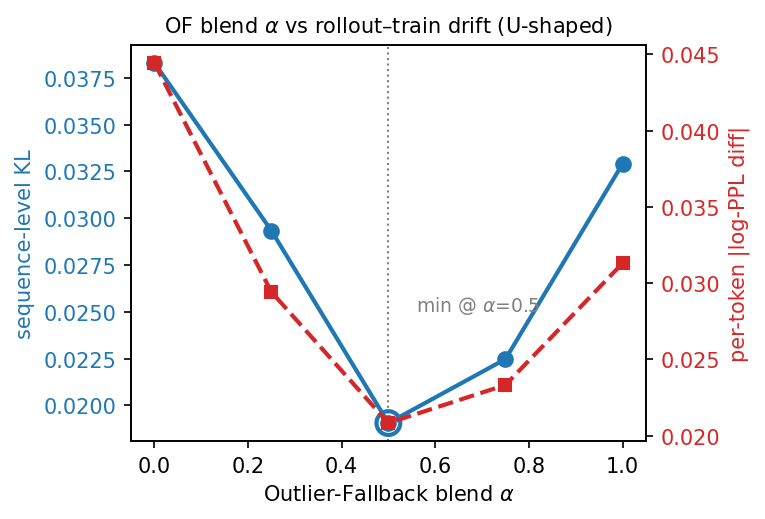}
  \caption{Outlier-Fallback blend $\alpha$ versus rollout--training drift (Qwen3-30B-A3B, controlled single-step isolation: OF only, MBS/AQN/attention-quant disabled, learning rate $0$, $32$ fixed DAPO-MATH prompts). Sequence-level KL between rollout and training log-probabilities, and per-token absolute log-perplexity difference, are both U-shaped in $\alpha$ with a clean minimum at $\alpha=0.5$. The training-dynamics counterpart of Table~\ref{tab:of_alpha}.}
  \label{fig:alpha_drift}
\end{figure}

\paragraph{Dense vs.\ MoE asymmetry.}
In MoE architectures, expert routing acts as a natural error-correcting code: even if one expert's small weights are pruned, other experts can compensate. Dense models have no such redundancy, so deadzone damage is fully expressed. This explains why OF is critical for dense ($+$17.5\,pp) but only modestly helpful for MoE ($+$1.5\,pp). Figure~\ref{fig:dense_moe_of} shows the training curves.

\begin{figure}[h]
  \centering
  \includegraphics[width=\linewidth]{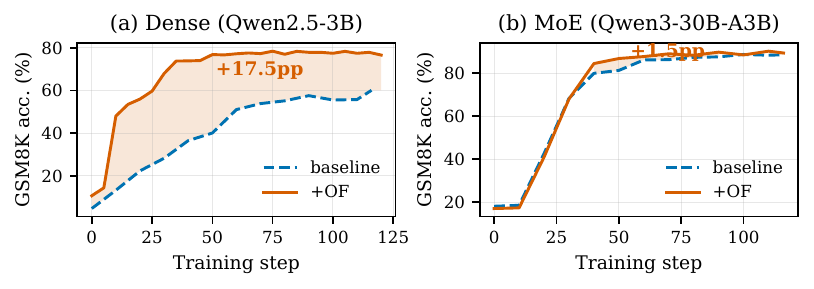}
  \caption{OF sensitivity by architecture. (a)~Dense: OF provides $+$17.5\,pp. (b)~MoE: OF provides only $+$1.5\,pp. MoE expert routing compensates for deadzone pruning.}
  \label{fig:dense_moe_of}
\end{figure}

\subsection{Pathway 2 details: policy gradient bias}

\paragraph{Gradient scale derivation.}
For a single linear layer with quantized weight, the forward pass output is $\hat{\by} = \gamma_b \cdot \bW_b\bx + \tilde{\be}_b\bx$, where $\tilde{\be}_b$ contains element-level errors. The STE backward pass treats $Q$ as identity, computing $\partial\cL/\partial\bW|_{\mathrm{STE}} = (\partial\cL/\partial\hat{\by})\cdot\bx^\top$ at a point shifted by $\gamma_b$. Through the chain rule across $L$ layers:
\begin{equation}
\hat{\nabla}_\theta\cL
  = \prod_{l=1}^L \gamma_{b(l)} \cdot \nabla_\theta\cL\big|_{\mathrm{true}}
  + \text{higher-order}\,.
\end{equation}

Since $\delta_b \sim \mathrm{Uniform}[0,1)$ approximately i.i.d., by the CLT:
\begin{equation}
  \sum_{l=1}^L \delta_{b(l)} \sim \cN(0,\, L/12)\,.
\end{equation}
For Qwen2.5-3B with $L{=}36$, $\mathrm{std} = \sqrt{36/12} = \sqrt{3} \approx 1.73$, so the gradient magnitude ratio ranges from ${\sim}0.18\times$ to ${\sim}5.6\times$ within one standard deviation.

\paragraph{Why MBS helps.}
MBS replaces $s_b = 2^e$ with $s_b^{\mathrm{MBS}} = 2^e\cdot(1 + c/256)$, where $c$ is an 8-bit mantissa correction at macro-block granularity (128 elements). This reduces $\Var(\gamma)$ by ${\sim}(256)^{-2}$, effectively eliminating scale-induced gradient bias.

\paragraph{Connection to TIS.}
TIS (Section~\ref{sec:setup}) clips the importance ratio $\rho_i$ to $\rho_{\max}$. In our experiments, TIS${}=5$ is critical for dense models: with TIS${}=2$, the quantized rollout policy diverges too far from the training policy, causing importance ratios to concentrate near the clipping boundary and effectively zeroing out the learning signal. The larger TIS${}=5$ threshold accommodates the additional policy divergence introduced by quantization. This is consistent with scale bias being a gradient-pathway problem: quantization-induced logit perturbations increase $\rho_i$ variance, and a tighter TIS clip amplifies the resulting gradient bias.

\subsection{Pathway 3 details: exploration--exploitation perturbation}
\label{app:eff_temp}

\paragraph{Grid noise distribution.}
The grid error for non-deadzone elements has variance depending on the local E2M1 step size:
\begin{equation}
  e_{b,i}^{\mathrm{grid}} \sim \mathrm{Uniform}[-\Delta(u_{b,i})/2,\; \Delta(u_{b,i})/2]\,,
\end{equation}
where $\Delta = 0.5$ for $|u|\in[0,2]$, $\Delta = 1$ for $|u|\in[2,4]$, and $\Delta = 2$ for $|u|\in[4,6]$.

\paragraph{Effective temperature.}
This noise propagates to logits as $\bell^Q = \bell + \boldeta$, $\boldeta \sim \cN(\bm{0}, \sigma_\eta^2\bm{I})$ approximately by CLT, where $\sigma_\eta^2 \propto \sum_l \sigma_{\delta,l}^2(\mathrm{grid})$ accumulates over layers. Matching the signal-to-noise ratio of logit differences under temperature scaling versus additive noise (derivation below) gives $T_{\mathrm{eff}} \approx \sqrt{1 + 2\sigma_\eta^2/\Var(\Delta\ell)} > 1$, directly increasing policy entropy:
\begin{equation}
  \cH[\pi_{Q(\btheta)}] > \cH[\pi_{\btheta} ].
\end{equation}
This matches the experimental observation: baseline entropy drops to 0.35 by step~50, while AQN+MBS entropy stays at 0.61.

\paragraph{Derivation of $T_{\mathrm{eff}}$ (Eq.~\ref{eq:eff_temp}).}
The square-root form arises from matching the softmax-induced action distribution under additive logit noise to the same distribution under deterministic temperature scaling. Three observations drive the derivation.

\emph{(i) Logit differences are all that matter.} The softmax policy depends only on pairwise differences: for actions $a, b$,
\begin{equation}
\log \frac{\pi(a\,|\,s)}{\pi(b\,|\,s)} \;=\; \ell_a - \ell_b \;=:\; \Delta\ell_{ab}.
\end{equation}
Temperature scaling at $T$ shrinks these differences uniformly to $\Delta\ell_{ab}/T$.

\emph{(ii) Per-pair noise has doubled variance.} With $\eta_a, \eta_b$ i.i.d.\ $\sim\cN(0,\sigma_\eta^2)$, the pairwise difference is
\begin{equation}
\Delta\eta_{ab} \;:=\; \eta_a - \eta_b \;\sim\; \cN(0,\, 2\sigma_\eta^2),
\end{equation}
which is the source of the factor 2 inside the square root of Eq.~\ref{eq:eff_temp}.

\emph{(iii) Probit--logit Gaussian-marginalization identity.} Using the classical approximation $\sigma(z) \approx \Phi(z/\kappa)$ with $\kappa = \sqrt{8/\pi}$, and the fact that the convolution of a probit with a Gaussian is another probit with summed variances,
\begin{equation}
\label{eq:probit_logit_match}
\Expect_{\Delta\eta_{ab}}\!\bigl[\sigma(\Delta\ell_{ab} + \Delta\eta_{ab})\bigr]
\;\approx\;
\sigma\!\left(\frac{\Delta\ell_{ab}}{\sqrt{1 + \Var(\Delta\eta_{ab})/\kappa^2}}\right)\,,
\end{equation}
where $\sigma(\cdot)$ is the logistic function and $\Phi(\cdot)$ the standard-normal CDF. The right-hand side is exactly a temperature-scaled deterministic softmax at $T = \sqrt{1 + 2\sigma_\eta^2/\kappa^2}$, recovering the square-root form.

\emph{From $\kappa^2$ to $\Var(\Delta\ell)$.}
Eq.~\ref{eq:eff_temp} substitutes the empirical logit-difference variance $\Var(\Delta\ell)$ for the constant $\kappa^2$. This is a signal-to-noise argument: $\kappa$ is the characteristic logit scale over which the sigmoid transitions, while $\Var(\Delta\ell)^{1/2}$ is the policy's \emph{own} characteristic logit scale at the current training step. The substitution makes $T_{\mathrm{eff}}$ adaptive to the current policy sharpness: (a)~when the policy is confident, $\Var(\Delta\ell) \gg \sigma_\eta^2$, so $T_{\mathrm{eff}} \to 1$ and noise is negligible; (b)~when the policy is flat, $\Var(\Delta\ell) \ll \sigma_\eta^2$, so $T_{\mathrm{eff}} \to \sqrt{2\sigma_\eta^2/\Var(\Delta\ell)}$ and noise dominates. The ``$1{+}$'' clamps $T_{\mathrm{eff}} \ge 1$: zero-mean grid noise can only \emph{widen}, never sharpen, the policy. In the noise-free limit $\sigma_\eta\to 0$, $T_{\mathrm{eff}} = 1$ recovers the unperturbed policy.

\paragraph{AQN as controlled exploration.}
AQN injects scheduled noise $\sigma_{\mathrm{AQN}}(t): \sigma_{\mathrm{start}} \to \sigma_{\mathrm{end}}$, giving total noise $\sigma_{\mathrm{total}}(t) = \sqrt{\sigma_{\mathrm{grid}}^2 + \sigma_{\mathrm{AQN}}^2(t)}$. Early in training, $\sigma_{\mathrm{AQN}} \gg \sigma_{\mathrm{grid}}$ provides controlled broad exploration; late, $\sigma_{\mathrm{AQN}} \to 0$ allows exploitation.

\paragraph{AQN+MBS synergy.}
Without MBS, the noise floor includes scale error (${\sim}54\%$ RMSE in the scale ratio, block-correlated). This structured noise contaminates AQN's controlled exploration. After MBS reduces $\sigma_{\mathrm{scale}}$ to ${\sim}0.1\%$, only the i.i.d.-like $\sigma_{\mathrm{grid}}$ remains, which AQN can effectively ``train through''. This explains the superadditive result: MBS alone $+$0.9\,pp, AQN alone $+$0.4\,pp, combined $+$2.0\,pp. Figures~\ref{fig:sigma_sweep} and~\ref{fig:grad_fluct} provide supporting evidence.

\begin{figure}[h]
  \centering
  \includegraphics[width=0.5\linewidth]{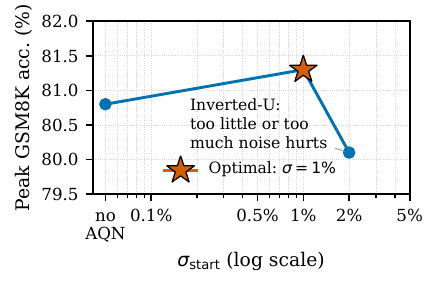}
  \caption{AQN $\sigma_{\mathrm{start}}$ sensitivity (Dense, MBS+OF). $\sigma = 1\%$ is optimal; $2\%$ overshoots and degrades below no-AQN baseline.}
  \label{fig:sigma_sweep}
\end{figure}

\section{Complementarity with upstream techniques}
\label{app:upstream_techniques}

Our two error corrections (MBS, OF) operate during quantization, while AQN operates on the training dynamics. An alternative strategy is to reshape the input distribution \emph{before} quantization so that the format's limitations bite less. Stochastic rounding (SR), originally introduced for low-precision neural training by~\citet{gupta2015limited}, and random Hadamard transforms (RHT), as used by~\citet{tseng2025trainingllmsmxfp4} for MXFP4 pre-training, both follow this upstream approach.

In terms of our decomposition, SR primarily targets $\be^{\mathrm{grid}}$: it makes the rounding error exactly zero-mean at every point (not just on average; classical quantization-noise theory~\citep{bennett1948,widrow2008quantization} already gives the additive-uniform model that SR enforces unconditionally), eliminating the local bias within each quantization bin. SR also softens the deadzone (a value below the threshold has a nonzero probability of rounding to the nearest nonzero grid point) but does not eliminate it. RHT primarily targets $\be^{\mathrm{DZ}}$: by spreading outliers across elements via a random orthogonal rotation, it equalizes magnitudes within each block, reducing the fraction of values that fall below the deadzone threshold $m_b/24$.

Crucially, neither SR nor RHT addresses $\be^{\mathrm{scale}}$ directly: the E8M0 power-of-two scale remains after any input rotation. MBS is the only technique in our toolkit (or theirs) that corrects this component. The two families of techniques are therefore \textbf{complementary}: SR/RHT reshape the distribution to reduce deadzone and grid error at the source, while MBS/OF/AQN correct whatever error remains after quantization. Exploring their combination is a promising direction for future work.

\section{Empirical validation of decomposition structure}
\label{app:empirical_orth}

We measure the pairwise cosine similarities and the cross term on real weight tensors to validate the MSE identity (Eq.~\ref{eq:identity}) and the anti-correlation (Remark~\ref{rem:anticorr}). For each linear layer's weight tensor, we compute the three error components (Definition~\ref{def:error_components}) and report:

\begin{enumerate}
\item \textbf{Identity verification}: confirm that $\norm{\be^{\mathrm{scale}}}^2 + \norm{\be^{\mathrm{DZ}}}^2 + \norm{\be^{\mathrm{grid}}}^2 + 2\inner{\be^{\mathrm{scale}}}{\be^{\mathrm{grid}}} = \norm{\be}^2$ holds exactly (up to floating-point precision) for each tensor.
\item \textbf{Exact orthogonality}: verify $\cos(\be^{\mathrm{scale}}, \be^{\mathrm{DZ}}) = 0$ and $\cos(\be^{\mathrm{DZ}}, \be^{\mathrm{grid}}) = 0$ exactly, confirming Lemma~\ref{lem:exact}.
\item \textbf{Anti-correlation}: report $\cos(\be^{\mathrm{scale}}, \be^{\mathrm{grid}})$ per layer; expected $\approx -0.66$ consistently (Remark~\ref{rem:anticorr}).
\end{enumerate}

The main-body Table~\ref{tab:decomp_empirical} reports aggregated results across two model scales, and the accompanying main-body Figure~\ref{fig:orthogonality} visualizes the distribution of pairwise cosine similarities across all weight tensors, confirming: $\cos(\be^{\mathrm{scale}}, \be^{\mathrm{DZ}}) = 0$ and $\cos(\be^{\mathrm{DZ}}, \be^{\mathrm{grid}}) = 0$ exactly (Lemma~\ref{lem:exact}), while $\cos(\be^{\mathrm{scale}}, \be^{\mathrm{grid}}) \approx -0.66$ with remarkably tight spread.

\section{Long chain-of-thought: scope and trajectory divergence}
\label{app:longcot}

Section~\ref{sec:pathway_exploration} argues that grid noise, being zero-mean, raises the policy's effective temperature but ``averages out over full response sequences.'' This holds for the short-to-medium reasoning of GSM8K but reaches its limit as chains lengthen. We document the boundary as a scope limitation, not a revision of the GSM8K result.

\paragraph{Per-step drift is task-independent; outcomes are not.}
Holding the recipe family fixed (MBS+OF / AQN+MBS+OF), the rollout--actor probability Pearson correlation is essentially identical on a short task (GSM8K, median ${\approx}\,280$-token responses) and a long one (DAPO-MATH, ${\approx}\,795$ tokens): $0.999$ in BF16 versus ${\approx}\,0.97$ in MXFP4 in both ($0.975$ / $0.971$, Figure~\ref{fig:pearson_corr}). The MXFP4 per-step drift is intrinsic to W4A4, not task-specific. Yet outcomes diverge: the short task \emph{exceeds} BF16 by $+1.0$\,pp (Table~\ref{tab:moe}), while on the DAPO-MATH long-response validation set the same recipe under-recovers substantially. The only variable separating recovery from failure is sequence length --- the floor does not cancel but \emph{compounds} under longer autoregressive generation.

\begin{figure}[h]
  \centering
  \includegraphics[width=0.95\linewidth]{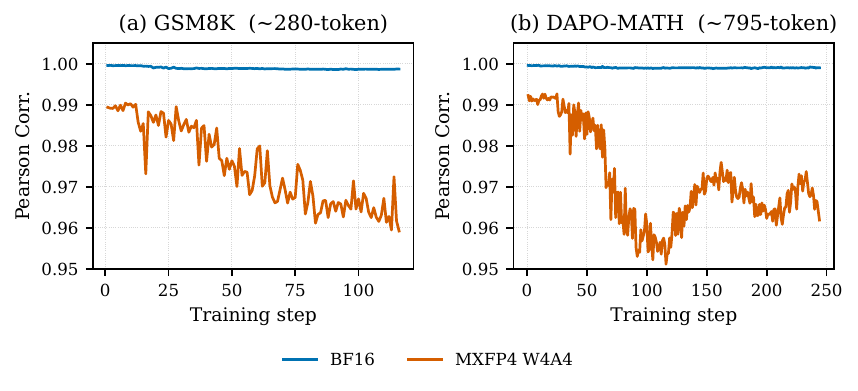}
  \caption{Per-step rollout--actor probability Pearson correlation across training, on a short task (a, GSM8K) and a long one (b, DAPO-MATH). BF16 stays at ${\approx}\,1.000$ on both tasks; MXFP4 W4A4 drifts to a mean of ${\approx}\,0.97$ on both ($0.975$ on GSM8K, $0.971$ on DAPO-MATH). The MXFP4 per-step drift is therefore intrinsic to W4A4 rather than task-specific --- yet end-task outcomes diverge sharply with response length (Table~\ref{tab:moe} for the short task; under-recovery on the long task). The same per-step number gives recovery in one regime and failure in the other.}
  \label{fig:pearson_corr}
\end{figure}

\paragraph{Mechanism: greedy trajectory fork.}
Under deterministic greedy decoding (no sampling noise, so any divergence is purely the quantization floor), BF16 and MXFP4 policies starting from the same base checkpoint produce near-disjoint trajectories (Figure~\ref{fig:greedy_fork}). A high teacher-forced per-token correlation looks benign because it is the quantity importance-sampling sees, but free generation occasionally flips an early arg-max, and that single flip compounds into a different solution path. After RL training the fork becomes systematically harmful: greedy disagreements are dominated by BF16-correct/MXFP4-wrong, with MXFP4 also producing many more unparseable or non-terminating responses. The long-CoT gap is therefore driven by trajectory divergence and non-termination, not by the static error budget our corrections minimize --- importance-sampling-side fixes cannot close it because reweighting tokens cannot undo the fact that the explored \emph{trajectories} differ.

\paragraph{Quality vs.\ response length.}
Under a strict $4$-gram repetition detector that excludes unformatted-but-valid outputs and length-capped truncation, W4A4 quality degrades earlier than BF16 with length, and degeneracy dominates the long tail (Figure~\ref{fig:quality_by_length}). The recovery the corrections achieve on short CoT does not extend to the long tail.

\paragraph{Where the gap is \emph{not}.}
In our setup, rollout and training forward passes already share the same QDQ path (identical E8M0 block scale, E2M1 grid, OF residual, rounding and block size), so the residual divergence is intrinsic to free-running autoregressive decoding rather than a rollout--training \emph{kernel} mismatch. The most promising next steps for long-CoT fidelity are therefore higher activation precision (e.g.\ W4A8) rather than rollout-kernel alignment.

\begin{figure}[h]
  \centering
  \includegraphics[width=0.7\linewidth]{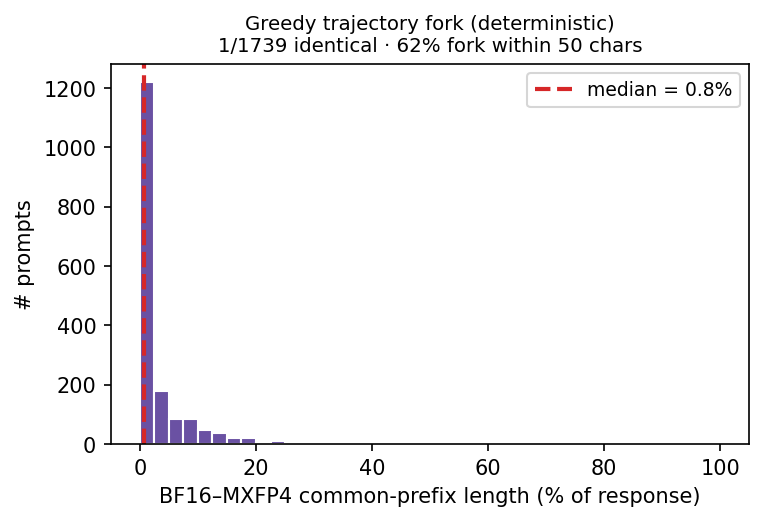}
  \caption{Greedy trajectory fork (deterministic decoding, no sampling noise; $1{,}739$ DAPO-test prompts, shared base checkpoint). Distribution of BF16-vs-MXFP4 common-prefix length: trajectories diverge early and are near-disjoint. Because greedy decoding is deterministic, every divergence is attributable to the quantization floor alone. On \emph{trained} checkpoints the fork becomes systematically harmful: greedy disagreements split asymmetrically toward BF16-correct/MXFP4-wrong, with MXFP4 also producing several-fold more unparseable / non-terminating outputs.}
  \label{fig:greedy_fork}
\end{figure}

\begin{figure}[h]
  \centering
  \includegraphics[width=0.95\linewidth]{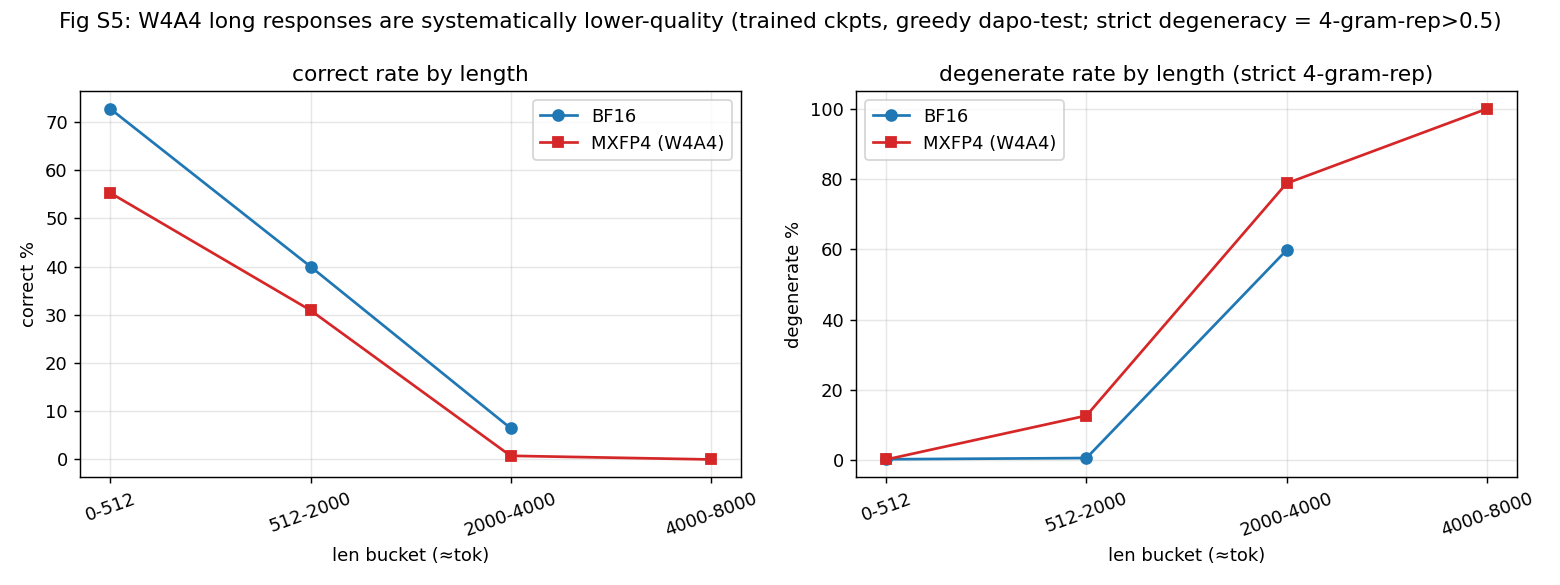}
  \caption{Quality vs.\ response length on the DAPO long-CoT validation set (trained checkpoints, greedy). \emph{Left:} correct rate by length bucket. \emph{Right:} degenerate rate under a strict $4$-gram repetition detector ($\mathrm{rep}>0.5$; excludes unformatted-but-valid outputs and length-capped truncation, which would otherwise inflate the count). W4A4 degenerates earlier and more, and dominates the long tail.}
  \label{fig:quality_by_length}
\end{figure}


\section{Training hyperparameters (MoE)}
\label{app:hyperparams}

We list every hyperparameter explicitly set in our MoE training script
\texttt{train\_moe\_w4a4\_aqn\_rht\_sr.sh}. All other knobs use the
verl-VeRL defaults.

\subsection*{GRPO}

\begin{itemize}\setlength{\itemsep}{1pt}
  \item Advantage estimator: \texttt{grpo}.
  \item \texttt{rollout.n} (rollouts per prompt): $2$.
  \item \texttt{train\_batch\_size} (prompts): $64$.
  \item \texttt{ppo\_mini\_batch\_size} (prompts): $32$.
  \item \texttt{ppo\_micro\_batch\_size\_per\_gpu}: $1$ (training and inference).
  \item \texttt{use\_dynamic\_bsz}: \texttt{True};
        \texttt{ppo\_max\_token\_len\_per\_gpu} $= 2048$.
  \item \texttt{max\_prompt\_length}: $1024$;
        \texttt{max\_response\_length}: $1024$.
  \item KL: \texttt{algorithm.use\_kl\_in\_reward=False},
        \texttt{algorithm.kl\_ctrl.kl\_coef=0.0};
        \texttt{actor.use\_kl\_loss=False}.
  \item PPO clip (token-level, asymmetric):
        $\epsilon_{\rm low}{=}0.2$,
        $\epsilon_{\rm high}{=}0.28$ (DAPO-style clip-higher),
        dual-clip $c{=}10.0$;
        \texttt{loss\_type=ppo\_clip}.
  \item Token-level Truncated Importance Sampling
        (TIS)~\citep{yao2025tis}:
        \texttt{rollout\_is=token},
        \texttt{rollout\_is\_threshold=2.0},
        \texttt{tis\_strategy=fixed},
        \texttt{bypass\_mode=false}.
  \item \texttt{loss\_agg\_mode=token-mean};
        \texttt{entropy\_coeff=0}.
  \item Reward: rule-based GSM8K grader (verl built-in,
        \texttt{reward\_model.enable=False} since rewards are computed
        from the dataset's ground-truth field).
  \item Optimizer:
        \texttt{lr=1e-6},
        \texttt{lr\_warmup\_steps=10},
        \texttt{lr\_decay\_style=constant},
        \texttt{weight\_decay=0.1},
        \texttt{clip\_grad=1.0}.
  \item \texttt{trainer.total\_epochs=1};
        \texttt{trainer.val\_before\_train=True};
        \texttt{trainer.test\_freq=10};
        \texttt{trainer.save\_freq=300}.
  \item Single seed per configuration; statistical significance
        limitation acknowledged in Section~\ref{sec:limitations}.
\end{itemize}

\subsection*{Rollout}

\begin{itemize}\setlength{\itemsep}{1pt}
  \item Engine: vLLM (\texttt{rollout.name=vllm});
        \texttt{enforce\_eager=True}; \texttt{free\_cache\_engine=True};
        \texttt{enable\_chunked\_prefill=True};
        \texttt{max\_num\_batched\_tokens=8192}.
  \item \texttt{tensor\_model\_parallel\_size=8} for inference.
  \item \texttt{gpu\_memory\_utilization=0.75}.
  \item Training rollout: \texttt{temperature=1.0},
        \texttt{top\_p=1.0}, \texttt{top\_k=-1}.
  \item Validation rollout (\texttt{val\_kwargs}):
        \texttt{temperature=1.0}, \texttt{top\_p=1.0},
        \texttt{top\_k=-1}, \texttt{do\_sample=True}, \texttt{n=1}.
  \item \texttt{calculate\_log\_probs=True};
        \texttt{enable\_rollout\_routing\_replay=True} (MoE router
        cached from rollout to actor).
\end{itemize}

\subsection*{Quantization (MXFP4 W4A4)}

\begin{itemize}\setlength{\itemsep}{1pt}
  \item Format: MXFP4 (E2M1 grid + E8M0 power-of-two block scale)
        \citep{rouhani2023ocp}; W4A4 — both weight and activation
        tensors entering each matmul are quantized; gradients remain
        BF16.
  \item Block size: $32$ along the input axis of each Linear.
  \item Quantizer path: native CUDA kernel
        (\texttt{actor.force\_py=False}); inference uses the same
        kernel via \texttt{rollout.prequant\_format=mxfp4}.
\end{itemize}

\subsection*{Corrections}

\paragraph{Macro Block Scaling (MBS)}~\citep{chhugani2026mxfp4}.
Macro-block size $128$ ($= 4$ contiguous MXFP4 sub-blocks of $32$);
$8$-bit mantissa correction (E0M8, $256$ levels) per macro-block;
applied identically to weights on both OF passes when MBS$+$OF is
enabled.

\paragraph{Outlier Fallback (OF)}~\citep{zhang2025fallback}.
Two-pass residual MXFP4 QDQ with $\alpha = 1.0$:
$q_1 = \mathrm{MXFP4}(\bw)$,
$q_2 = \mathrm{MXFP4}(\bw - q_1)$,
output $q_1 + q_2$.

\paragraph{Adaptive Quantization Noise (AQN)} ~\cite{huang2025qerl}.
Per-rollout Gaussian additive noise injection on weights with
exponentially decaying magnitude:
\begin{itemize}\setlength{\itemsep}{1pt}
  \item \texttt{trainer.aqn.enabled=True}.
  \item \texttt{sigma\_start=0.01}.
  \item \texttt{sigma\_end=0.001}.
  \item \texttt{num\_stages=10}.
  \item \texttt{target\_patterns=['layernorm']}.
  \item \texttt{target\_multipliers.post\_attention\_layernorm=1.414}
        (i.e.\ $\sqrt{2}$, compensating for the magnitude doubling at
        the residual-stream branch).
\end{itemize}

\subsection*{System / parallelism}

\begin{itemize}\setlength{\itemsep}{1pt}
  \item Model: \texttt{Qwen3-30B-A3B-Base}.
  \item Backend: Megatron-Core (via verl-VeRL,
        \texttt{config-name=ppo\_megatron\_trainer}); custom
        \texttt{layers.py} (the ACS replacement) provides MXFP4 + RHT
        + SR for attention matmuls; \texttt{te\_grouped\_quant\_patch.py}
        provides MXFP4 for grouped MoE GEMMs.
  \item Parallelism (training): \texttt{TP=1}, \texttt{PP=1},
        \texttt{VPP=null}, \texttt{CP=1}, \texttt{EP=8},
        \texttt{ETP=1}.
  \item Megatron MoE knobs: \texttt{moe\_grouped\_gemm=True},
        \texttt{moe\_token\_dispatcher\_type=flex},
        \texttt{moe\_router\_dtype=fp32},
        \texttt{moe\_enable\_deepep=True},
        \texttt{moe\_permute\_fusion=False}.
  \item Memory: \texttt{param\_offload=True},
        \texttt{grad\_offload=True},
        \texttt{optimizer\_offload=True}
        (\texttt{optimizer\_offload\_fraction=1.0});
        \texttt{use\_precision\_aware\_optimizer=True};
        gradient checkpointing enabled.
  \item Checkpointing: \texttt{use\_mbridge=True},
        \texttt{use\_dist\_checkpointing=False}.
  \item Distributed setup: \texttt{trainer.n\_gpus\_per\_node=8},
        \texttt{trainer.nnodes=1}.
  \item Router replay: \texttt{router\_replay.mode=R3} (replay rollout
        routing decisions during the actor update).
  \item LoRA: disabled (\texttt{lora.rank=0}; full-parameter training).
\end{itemize}

\subsection*{Data}

\begin{itemize}\setlength{\itemsep}{1pt}
  \item Training: GSM8K \texttt{train.parquet} ($7{,}473$ problems).
  \item Validation: GSM8K \texttt{test.parquet} ($1{,}319$ problems).
  \item Preprocessing: parquet via \texttt{verl.data\_preprocess.gsm8k};
        prompt key \texttt{prompt}; truncation mode \texttt{error}
        (longer-than-budget prompts abort the job rather than silently
        truncate).
\end{itemize}


\end{document}